\documentclass{article}

\usepackage{microtype}
\usepackage{graphicx}
\usepackage{subfigure}
\usepackage{booktabs} % for professional tables

\usepackage{hyperref}

\usepackage[accepted]{icml2023-diffxyz}

\usepackage{amsmath}
\usepackage{amssymb}
\usepackage{mathtools}
\usepackage{amsthm}

\usepackage[capitalize,noabbrev]{cleveref}

\theoremstyle{plain}
\newtheorem{theorem}{Theorem}[section]
\newtheorem{proposition}[theorem]{Proposition}
\newtheorem{lemma}[theorem]{Lemma}

\theoremstyle{definition}

\theoremstyle{remark}

\usepackage{sg-macros}
\usepackage{dblfloatfix}

\usepackage{pgfplots}
\pgfplotsset{compat=1.5, every axis/.append style={font=\small, /pgf/number format/1000 sep={}}, every mark/.append style={solid},}
\usepackage{tikz}
\usetikzlibrary{quotes, arrows.meta, angles, math, calc, 3d, shapes, intersections, plotmarks, positioning} %,external
\usepackage{tikzscale}

\icmltitlerunning{Towards Understanding Gradient Approximation in Deep Declarative Networks}

\begin{document}

\twocolumn[
\icmltitle{Towards Understanding Gradient Approximation in\\Equality Constrained Deep Declarative Networks}



\icmlsetsymbol{equal}{*}

\begin{icmlauthorlist}
\icmlauthor{Stephen Gould}{ANU}
\icmlauthor{Ming Xu}{ANU}
\icmlauthor{Zhiwei Xu}{ANU}
\icmlauthor{Yanbin Liu}{ANU}
\end{icmlauthorlist}

\icmlaffiliation{ANU}{Australian National University, Canberra, Australia}
\icmlcorrespondingauthor{Stephen Gould}{stephen.gould@anu.edu.au}

\icmlkeywords{Deep Declarative Networks}

\vskip 0.3in
]



\printAffiliationsAndNotice{}  % leave blank if no need to mention equal contribution

\begin{abstract}
We explore conditions for when the gradient of a deep declarative node can be approximated by ignoring constraint terms and still result in a descent direction for the global loss function. This has important practical application when training deep learning models since the approximation is often computationally much more efficient than the true gradient calculation. We provide theoretical analysis for problems with linear equality constraints and normalization constraints, and show examples where the approximation works well in practice as well as some cautionary tales for when it fails.
\end{abstract}

% --- Introduction ---------------------------------------------------------------------------

\vspace{-6mm}
\section{Introduction}
\vspace{-1mm}
\label{sec:intro}

This paper investigates certain approximations to gradient calculations for differentiable constrained optimization problems. Our focus is on continuous optimizations problems that may be embedded within deep learning models~\cite{Gould:TR2016, Amos:ICML2017, Agrawal:NIPS19, Gould:PAMI2021, Blondel:NIPS22}. This is in contrast to works that compute search directions for back-propagating through discrete optimization problems where a true gradient does not exist or is uninformative (i.e., zero almost everywhere), e.g.,~\cite{Blondel:ICML20, Berthet:NIPS20, Vlastelica:ICLR20, Peterson:NIPS22}.

For continuous constrained optimization problems the gradient of a solution with respect to parameters of the problem (i.e., inputs) can be determined by implicit differentiation of the problem's optimality conditions~\cite{Amos:ICML2017, Agrawal:NIPS19, Gould:PAMI2021}. One of the main computational difficulties in the presence of constraints is evaluating quantities of the form $(A H^{-1}A\transpose)^{-1}$ where $A$ encodes first derivatives of the constraint functions and $H$ encodes second derivatives of the objective and constraints.

\citet{Gould:OTSDM2022} observed that for deep models involving optimal transport---a well-known differentiable optimization problem---ignoring the constraints in the backward pass, i.e., treating the problem as if it were unconstrained, still allows the model to learn while greatly speeding the backward pass. This prompts the question explored in this paper: why and when does this gradient approximation work?

% --- Main -----------------------------------------------------------------------------------

\section{Gradient Approximation}
\label{sec:grad_approx}

In this section we develop theoretical insights for when back-propagating through a differentiable optimization problem using an approximate gradient gives a descent direction for the global loss. Full proofs can be found in the appendix.

The following result for the derivative of the solution to parametrized equality constrained optimisation problems comes from~\citet{Gould:PAMI2021}[Prop.~4.5].

\begin{proposition}
	\emph{\cite{Gould:PAMI2021}.}
	Consider functions $f: \reals^n \times \reals^m \to \reals$ and $h: \reals^n \times
	\reals^m \to \reals^p$. Let
	\begin{align*}
		y(x) \in \left\{ \begin{array}{lll}
			\text{argmin}_{u \in \reals^m} & f(x, u) &\\
			\text{subject to} & h_i(x, u) = 0, & i = 1, \ldots, p
		\end{array} \right\}.
	\end{align*}
	Assume that $y(x)$ exists, that $f$ and $h = [h_1,\dots,h_p]\transpose$ are
	$2^{\text{nd}}$-order differentiable in the neighborhood of $(x, y(x))$,
	and that $\text{rank}(\dd[Y] h(x, y)) = p$.
	Then for $H$ non-singular
	\begin{align*}
		\dd y(x) &= H^{-1} \! A\transpose \!\! \left( A H^{-1} \! A\transpose \right)^{\!-1} \!\!\left(A H^{-1} B - C\right)  - H^{-1} B
	\end{align*}
	where
	$A = \dd[Y] h(x,y) \in \reals^{p \times m}$,
	$B = \dd[XY]^2 f(x, y) - \sum_{i = 1}^{p} \lambda_i \dd[XY]^2 h_i(x, y) \in \reals^{m \times n}$,
	$C = \dd[X] h(x,y) \in \reals^{p \times n}$, 
	$H = \dd[YY]^2 f(x, y) - \sum_{i = 1}^{p} \lambda_i \dd[YY]^2 h_i(x, y) \in \reals^{m \times m}$, 
	and $\lambda \in \reals^p$ satisfies $\lambda\transpose \!A = \dd[Y] f(x, y)$.
	\label{prop:ddnmain}
\end{proposition}
Symbol $D$ denotes the total or partial (with respect to the subscripted variable) derivative operator. We refer the reader to \citet{Gould:PAMI2021} for the full derivation.

Given an incoming gradient of a loss with respect to the output (i.e., solution) $\dd \Ell(y)$, back-propagation computes the gradient of the loss with respect to the input $x$ via the chain rule of differentiation as $\dd \Ell(x) = \dd \Ell(y) \dd y(x)$. As mentioned, however, terms involving $A$, namely $( A H^{-1} \! A\transpose)^{-1}$, may present significant computational challenges. Ignoring such terms gives a computationally much simpler expression, but how well does it approximate the true gradient?

Formally, define $\widehat{H} =  \dd[YY]^2 f(x, y)$ so that $H = \widehat{H} -  \sum_{i = 1}^{p} \lambda_i \dd[YY]^2 h_i(x, y)$ for a constrained problem. Let $v\transpose = \dd \Ell(y) \in \reals^{1 \times m}$ be the incoming gradient of the loss $\Ell$ with respect to output $y$, let $g\transpose = v\transpose \dd y(x)$ be the true gradient of the loss with respect to input $x$ and let $\widehat{g}\transpose = v\transpose \widehat{\dd y}(x) = -v\transpose \widehat{H}^{-1} B$ be the approximation obtained by ignoring constraints. We wish to understand when $-\widehat{g}$ is a descent direction for $\Ell$, i.e., when is
\begin{align}
	g\transpose \widehat{g} \geq 0 \,?
\end{align}
To simplify analysis and make progress towards some theoretical insights we will assume a single constraint function $h(u) = 0$ that is independent of $x$. Furthermore, we will assume that the objective function takes the special form $f(x, u) = x\transpose u + \tilde{f}(u)$. An example of this is the objective function for the optimal transport problem. Together, these assumptions imply that $C = 0$ and $B = I$ in \propref{prop:ddnmain}.

Substituting for $\dd y(x)$ and $\widehat{\dd y}(x)$ under these assumptions we have that $-\widehat{g}$ is a descent direction if and only if,
\begin{align}
	v\transpose \! \left(\! H^{-1} - \frac{H^{-1} a a\transpose H^{-1}}{a\transpose H^{-1} a}\right) \! \widehat{H}^{-1} v &\geq 0,
	\label{eqn:descentcond}
\end{align}
where we have written $a\transpose = A = \dd[Y] h(y) \in \reals^{1 \times m}$ to make it clear that we are only considering problems with a single constraint.\footnote{We recognize that for a single constraint the quantity $a\transpose H^{-1} a$ is trivial to invert and hence the approximation here offers little computational advantage. Nevertheless, as we will see, analysis from this simplification is instructive for more general settings.} We now explore two special cases.

\subsection{Special Case: Linear Constraints}

Consider the case of a single linear equality constraint, $a\transpose u = d$. In this case we have $\dd[YY]^2 h(u) = 0$ and therefore $H = \widehat{H}$. The condition that our approximate gradient $\widehat{\dd y}(x) = -H^{-1}$ always leads to a descent direction is
\begin{align}
	\min_{w} \, w\transpose \left(I - \frac{a a\transpose H^{-1}}{a\transpose H^{-1}a} \right) w &\geq 0
\end{align}
which holds if and only if\footnote{See \appref{sec:lindescentcond} for complete derivation.}
\begin{align}
	\max_{\|w\| = 1} \, w\transpose \left( \frac{a a\transpose H^{-1}}{a\transpose H^{-1}a} \right) w &\leq 1
	\label{eqn:lindescentcond}
\end{align}
where we have written $w = H^{-1}v$ from \eqnref{eqn:descentcond}.

Unfortunately this is only true when $\text{cond}(H) = 1$ as the following proposition shows.

\begin{proposition} \label{prop:badnews}
	Let $H \in \reals^{m \times m}$ be a non-singular symmetric matrix and let $a$ be an arbitrary vector in $\reals^n$. Then,
	\begin{align*}
		1 
		\leq 
		\max_{\|w\| = 1} \, w\transpose \left(\frac{aa\transpose H^{-1}}{a\transpose H^{-1} a}\right) w
		\leq
		\frac{1}{2} + \frac{\textup{cond}(H)}{2}.
	\end{align*}
\end{proposition}
%\begin{proof}
%	See \appref{sec:proof_badnews}.
%\end{proof}

The lower bound is bad news. It states that, in general, we cannot guarantee that the approximation will be a descent direction for all incoming loss gradients (unless $H \propto I$). But let us not despair. This is in the worst case. The next result concerns the expected value of $g\transpose \widehat{g}$ and tells us that, if $H^{-1}v$ is isotropic Gaussian distributed, then $-v\transpose \widehat{\dd y}(x)$ is a descent direction of the loss on average.
\begin{proposition}
	\label{prop:single_lin_expect}
	Let $w \sim \N(0, I)$. Then
	\begin{align*}
		\expectation{w\transpose \! \left(I - \frac{aa\transpose H^{-1}}{a\transpose H^{-1} a}\right) \! w}
		= m - 1 \geq 0.
	\end{align*}
\end{proposition}
%\begin{proof}
%	See \appref{sec:proof_single_lin_expect}.
%\end{proof}

The result can be extended to multiple ($1 \leq p \leq m$) linear equality constraints $Au = d$ as follows.
\begin{proposition}
	\label{prop:multi_lin_expect}
	Let $w \sim \N(0, I)$. Then
	\begin{align*}
		\expectation{w\transpose \! \left(I - A\transpose \! \left(A H^{-1} A\transpose\right)^{-1} \! A H^{-1}\right) \! w} &= m - p \geq 0.
	\end{align*}
\end{proposition}
%\begin{proof}
%	See \appref{sec:proof_multi_lin_expect}.
%\end{proof}

This result is encouraging: for linear equality constrained problems we can expect the approximate gradient to be a descent direction. Next we turn our attention to a non-linear constraint where the story is not as straightforward.

\subsection{Special Case: Normalization Constraint}

We now consider the case of a single non-linear constraint, the normalization constraint, $\|u\|^2 = 1$, which occurs in many problems such as projection onto the $L_2$-sphere and eigen decomposition.

Once again, let $\widehat{H} = \dd[YY]^2 f(x, y)$ and $H = \dd[YY]^2 f(x, y) - \lambda \dd[YY]^2 h(y) = \widehat{H} - \lambda I$. We will assume that $\widehat{H}^{-1}$ and $H^{-1}$ exist.\footnote{This implies, in particular, that $\lambda$ is not an eigenvalue of $\widehat{H}$, which is clearly not true for eigen decomposition (where we also have $B \neq I$). Still, some useful insights can be gained. A similar argument may be possible using pseudo-inverses or going back to the optimality conditions and deriving gradients directly.} Here we have $a \propto y$ so the general condition for the approximate gradient $\widehat{g} = -v\transpose \widehat{H}^{-1}$ to be a descent direction is
\begin{align}
	v\transpose \left(H^{-1} - \frac{H^{-1} y y\transpose H^{-1}}{y\transpose H^{-1} y}\right) \widehat{H}^{-1} v &\geq 0.
\end{align}
The left-hand side represents $g\transpose \widehat{g}$. As for the linear equality constrained case, we can compute its expected value.
\begin{proposition}
	\label{prop:norm_expect}
	Let $H^{-1}v \sim \N(0, I)$ and other quantities as defined above for the normalization constrained special case. Then
	\begin{align*}
		\expectation{g\transpose \widehat{g}}
		&= \sum_{i=1}^{m} \frac{\lambda_i - \lambda}{\lambda_i} - \frac{y\transpose \widehat{H}^{-1} y}{y\transpose H^{-1} y}
	\end{align*}
	where $\lambda_1 \leq \lambda_2 \leq \cdots \leq \lambda_m$ are the eigenvalues of $\widehat{H}$.
\end{proposition}
%\begin{proof}
%	See \appref{sec:proof_norm_expect}.
%\end{proof}

The above result is for general Hessian matrix $\widehat{H}$ and arbitrary $\lambda$. Let us consider two important (non-exhaustive) cases to give concrete bounds.

\begin{proposition}
	\label{prof:norm_expect_cases}
	Let $\widehat{H} \succ 0$, and let $g$ and $\widehat{g}$ be the true and approximate gradients, respectively, as defined above.

	\begin{minipage}{\columnwidth} \begin{itemize}
		\denselist \centering
		\item[(i)] If $\lambda < \lambda_1$, then $\expectation{g\transpose \widehat{g}} \geq 0$;
		\item[(ii)] If $\lambda > \lambda_m$, then $\expectation{g\transpose \widehat{g}} \leq 0$,
	\end{itemize} \end{minipage}

	\noindent
	where $\lambda_1$ and $\lambda_m$ are the smallest and largest eigenvalues of $\widehat{H}$, respectively.
\end{proposition}
%\begin{proof}
%	See \appref{sec:proof_norm_expect_cases}.
%\end{proof}

In summary, for the former case $-\widehat{g}$ is a descent direction on average, whereas for the latter case it is an ascent direction! Analogous results hold for $\widehat{H} \prec 0$.

% --- Experiments -----------------------------------------------------------------------------

\section{Examples and Experiments}
\label{sec:experiments}

In this section we experimentally validate the findings from above on three different optiimization problems. Our experimental setup is depicted in \figref{fig:experimental_setup}. Briefly, a data generating network provides input for a differentiable optimization problem. We train the data generating network so that the solution of the optimization problem matches some predetermined target. Further details are provided in \appref{sec:experimental_details}.

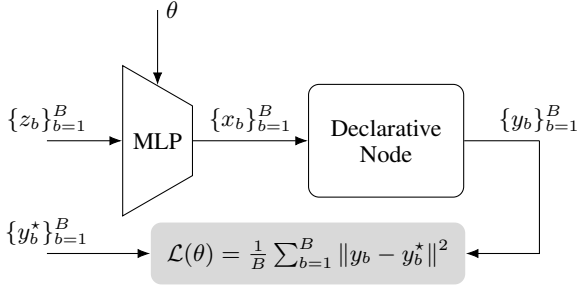
\begin{figure}
	\centering \small
	\begin{tikzpicture}
		\node[draw, rectangle, rounded corners, minimum height=1.5cm, inner sep=8pt, text width=1.5cm, align=center] (b) at (0,0) {Declarative\\Node};
		\node[draw, trapezium, rotate=-90, minimum width=2cm] (a) at ($(b.west) - (2,0)$) {};
		\node[] at (a) {MLP};
		\draw[-{Latex[]}, black] ($(a.south) - (1, 0)$) -- (a) node[pos=0, above] {$\{z_b\}_{b=1}^{B}$};
		\draw[-{Latex[]}, black] (a) -- (b) node[midway, above] {$\{x_b\}_{b=1}^{B}$};
		\draw[black] (b) -- ($(b.east) + (1.0, 0)$) node[pos=1, above] {$\{y_b\}_{b=1}^{B}$};
		
		\draw[-{Latex[]}, black] ($(a.west) + (0, 1)$) -- (a) node[pos=0, right] {$\theta$};
		
		\node[rounded corners, fill=black!15, inner sep=6pt] (c) at ($(b.west) - (0.0, 1.5)$) {$\Ell(\theta) = \frac{1}{B} \sum_{b=1}^{B} \|y_b - y_b^\star\|^2$};
		\draw[-{Latex[]},black] ($(a.south) - (1, 1.5)$) -- (c) node[pos=0, above] {$\{y^\star_b\}_{b=1}^{B}$};
		\draw[-{Latex[]},black] ($(b.east) + (1.0, 0)$) -- ($(b.east) + (1.0, -1.5)$) -- (c);
		
	\end{tikzpicture}
	\caption{Common experimental setup to compare behavior of approximate and exact gradients of constrained differentiable optimisation problems in a deep declarative network. Training data is a batch of randomly sampled input-target pairs $(z_b, y_b^\star) \in \reals^d \times \Y$. The input $z_b$ passes through a multi-layer perceptron to generate the parametrization $x_b$ for a  declarative node whose output (i.e., optimal value) is $y_b$. Thus $y_b$ is ultimately a function of the input $z_b$ and network parameters $\theta$. Training aims to adjust $\theta$ so as to minimize the square difference between output $y_b$ and target $y_b^\star$.}
	\label{fig:experimental_setup}
\end{figure}

\subsection{Euclidean Projection onto $L_2$-sphere}

Let us start with the simple problem of projecting a point $x \in \reals^n$ onto the unit sphere,
\begin{align}
	y(x) \in \left\{ \begin{array}{ll}
		\text{argmin} & \frac{1}{2} \|u - x\|^2 \\
		\text{subject to} & \|u\|_2 = 1
	\end{array} \right\}.
\end{align}
Here we have closed-form solution, $y = \frac{1}{\|x\|} x$,
with true and expected gradients given by
\begin{align}
	\dd y(x) = \frac{1}{\|x\|} \left(I - yy\transpose\right)
	\quad \text{and} \quad
	\widehat{\dd y}(x) = I.
\end{align}
The approximate gradient always gives a descent direction (when $\dd y(x)$ exists) since $I - yy\transpose$ is positive semidefinite.

Experimental results in \figref{fig:proj} confirm that the approximate gradient is always a descent direction, i.e., $g\transpose \widehat{g} > 0$, (bottom plots), and appears to work well for learning the parameters of the MLP especially during early iterations (top plots).

\begin{figure}
	\centering \small
	\setlength{\tabcolsep}{2pt}
	\begin{tabular}{cc}
		\includegraphics[width=0.24\textwidth]{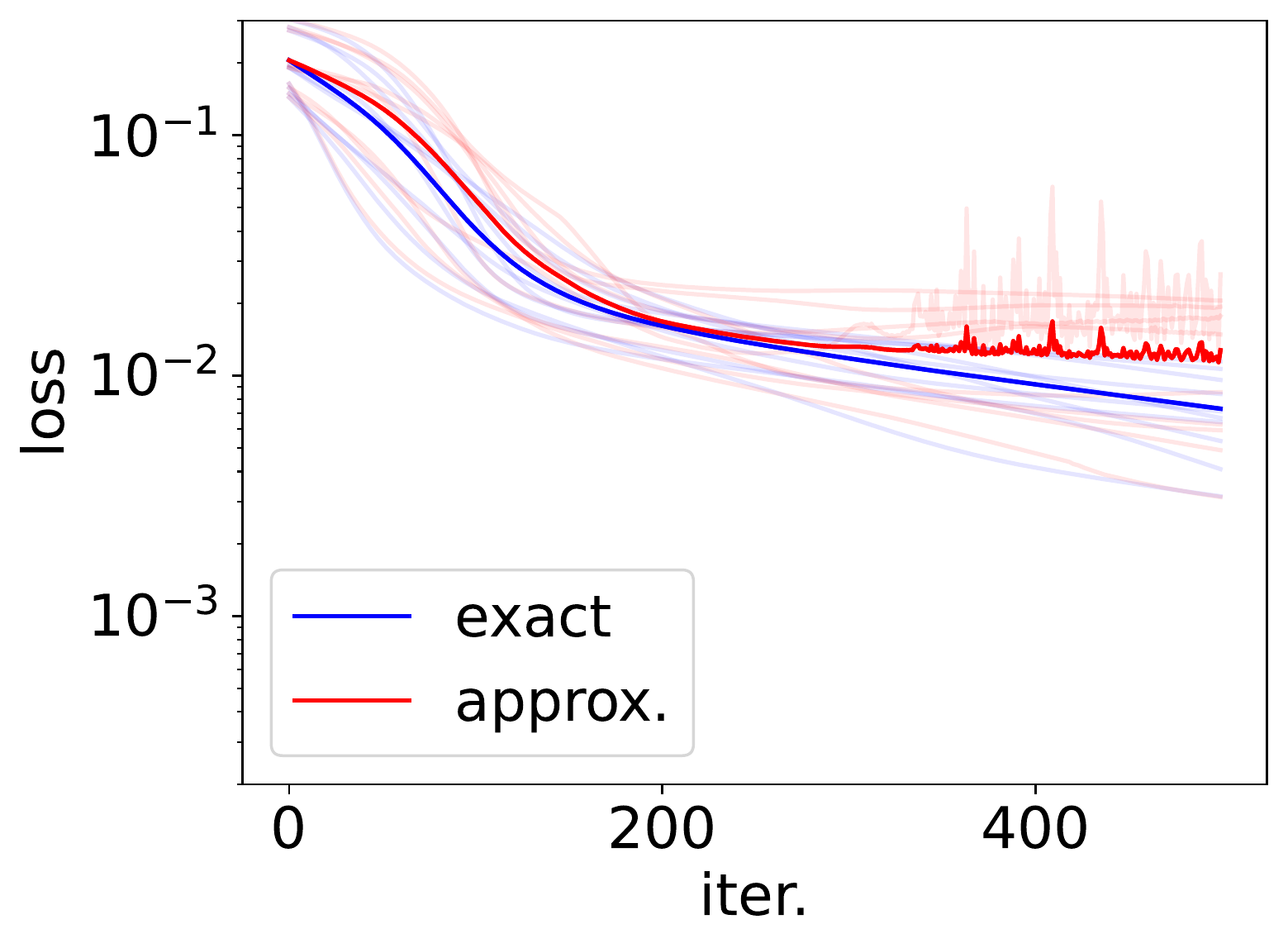} &
		\includegraphics[width=0.24\textwidth]{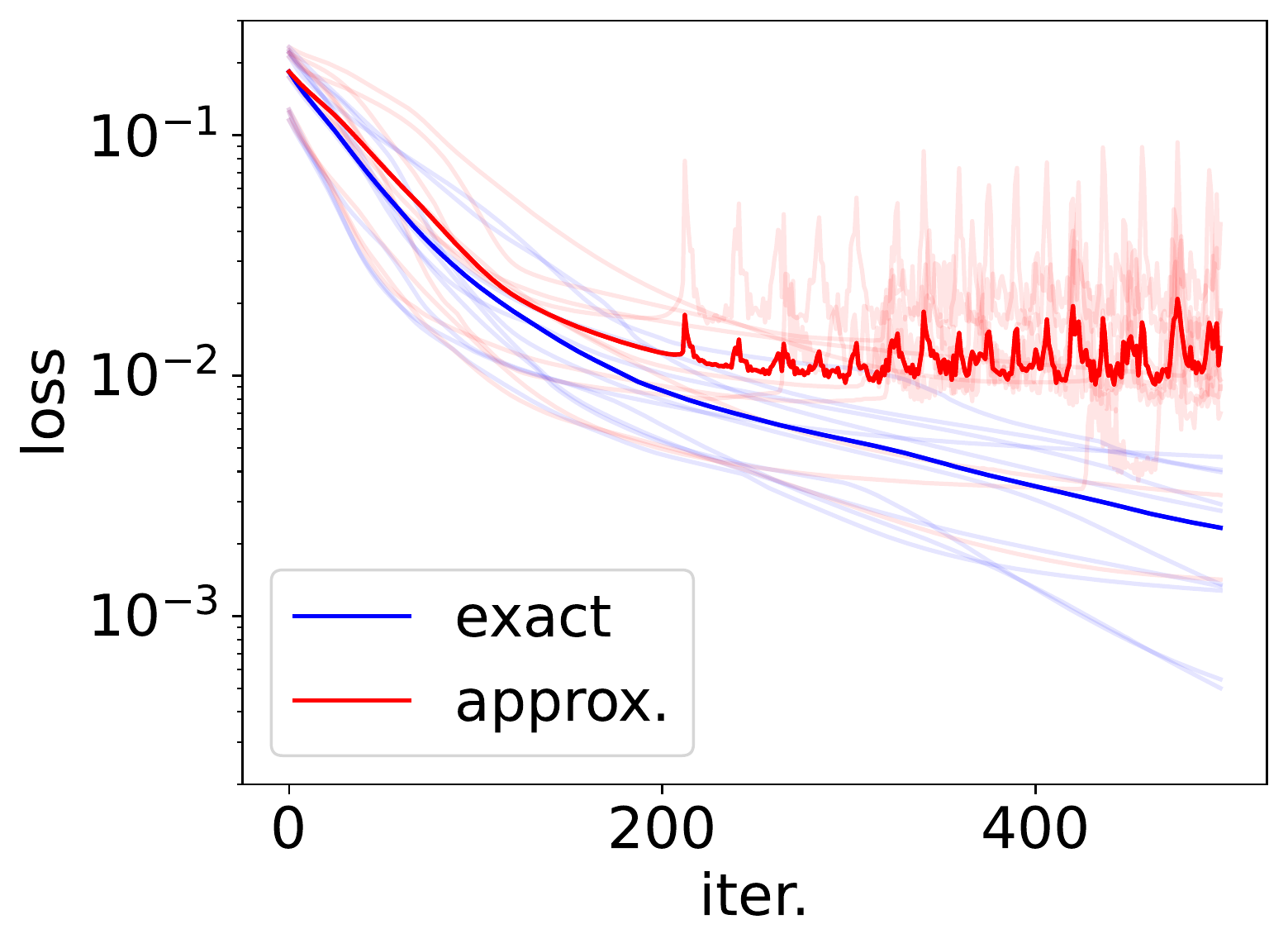} \\
		\includegraphics[width=0.24\textwidth]{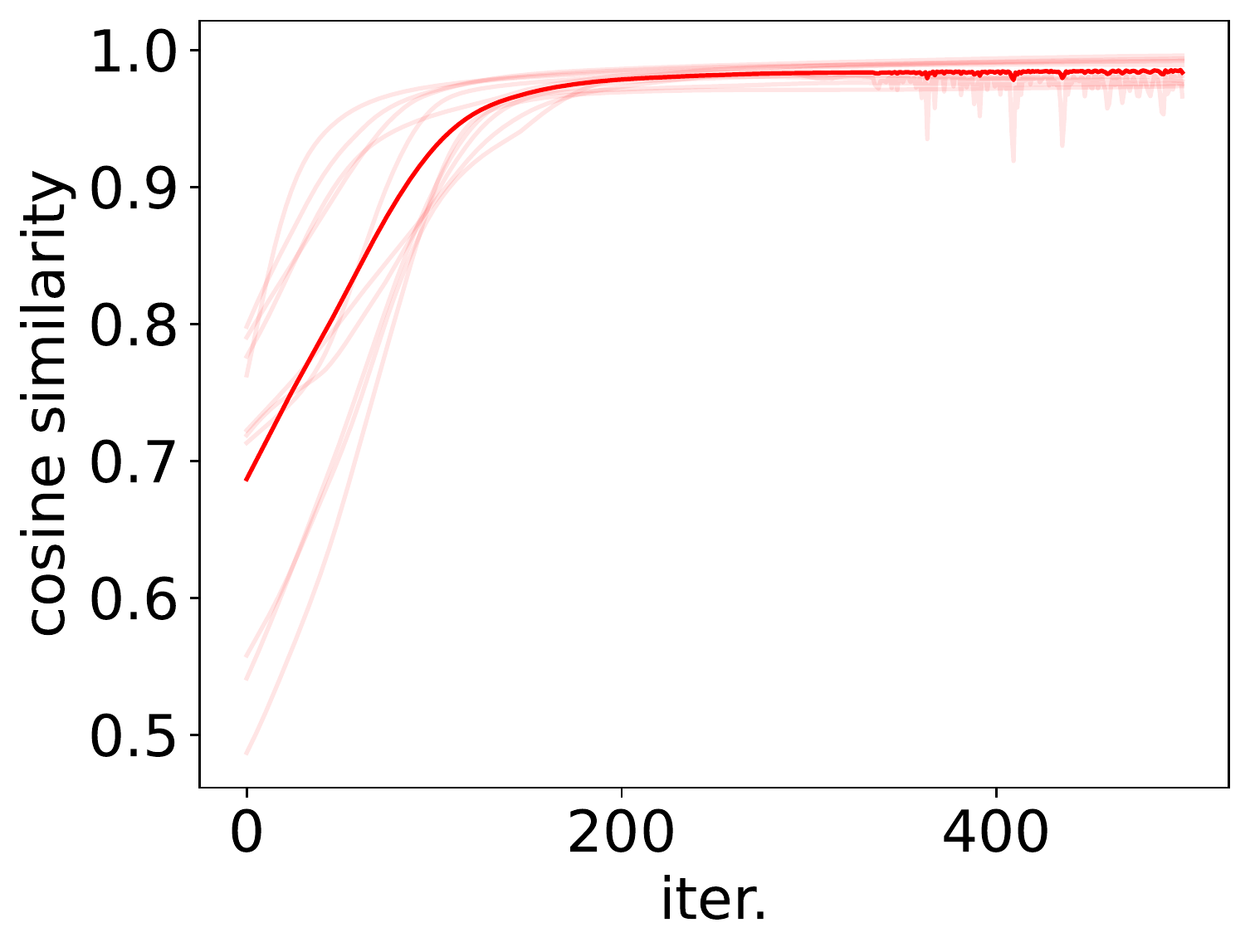} &
		\includegraphics[width=0.24\textwidth]{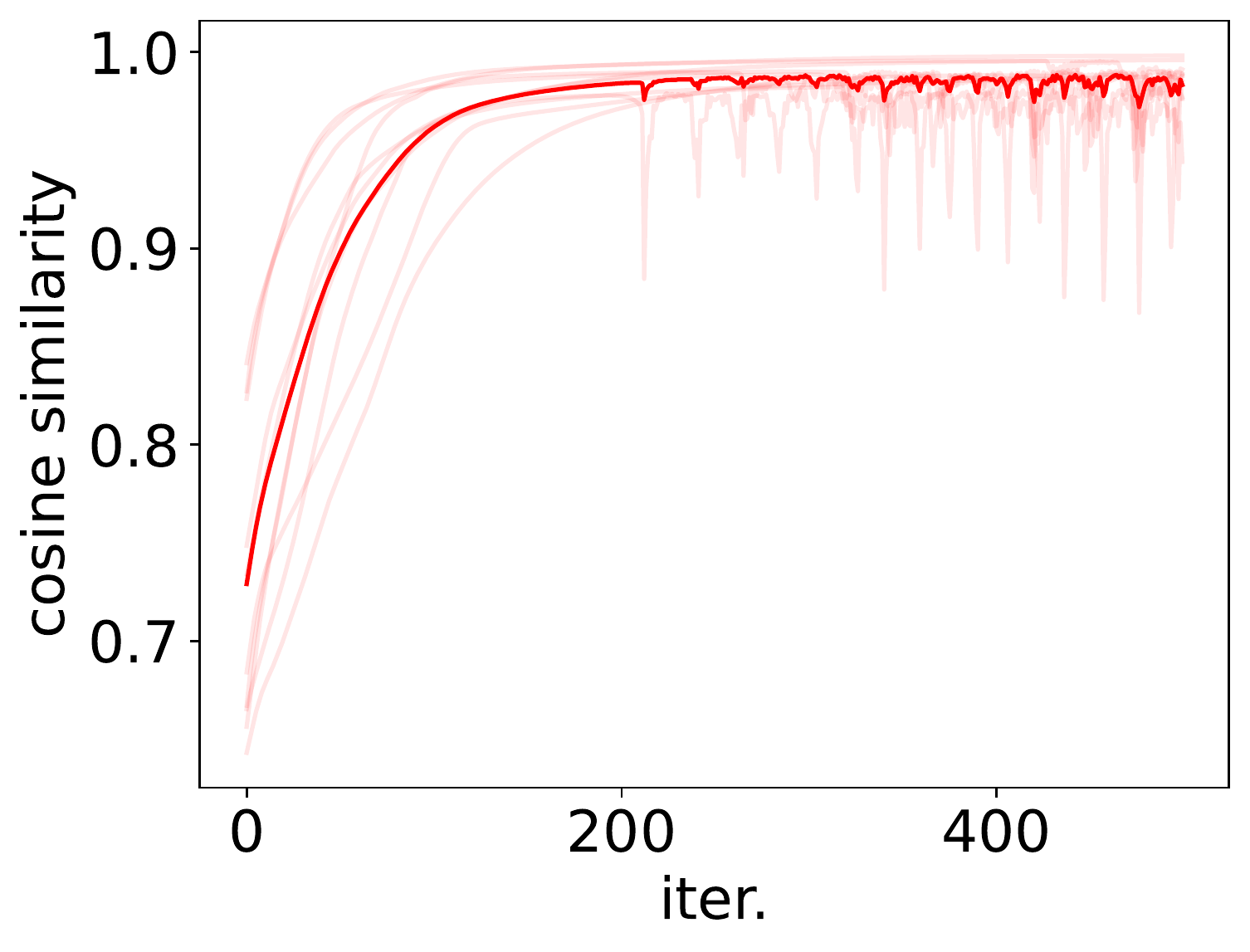} \\
		(a) under parameterized & (b) over parameterized
	\end{tabular}%
	\caption{Learning curves (top) for exact and approximate gradients for projection onto the unit sphere experiments. Bottom curves show cosine similarity between approximate and exact gradients for each point on the approximate learning curve. Left versus right shows low- versus high-dimensional $z_b$, respectively.}
	\label{fig:proj}
\end{figure}

\subsection{Optimal Transport}

Entropy regularized optimal transport is a linear equality constrained optimization problem~\cite{Cuturi:NIPS2013},
\begin{align}
	\begin{array}{ll}
		\text{minimize} & \langle P, M\rangle + \frac{1}{\gamma} \text{KL}(P \| rc\transpose) \\
		\text{subject to} & P 1 = r \text{ and } P\transpose 1 = c,
	\end{array}
	\label{eqn:op}
\end{align}
over variable $P \in \reals^{m \times n}_{+}$, where $M \in \reals^{m \times n}$ is an input cost matrix, $r$ and $c$ are positive vectors of row and column sum constraints (with $\ones\transpose r = \ones\transpose c$). Hyper-parameter $\gamma > 0$ controls the strength of the regularization term.

Typical learning curves and gradient similarity per iteration is shown in \figref{fig:ot}, depicting behavior much like the previous example---the approximate gradient is always a descent direction and works especially well during the early stages of training. This is consistent with our analysis.

\begin{figure*}
	\centering \small
	\setlength{\tabcolsep}{2pt}
	\begin{tabular}{cccc}
		\includegraphics[width=0.24\textwidth]{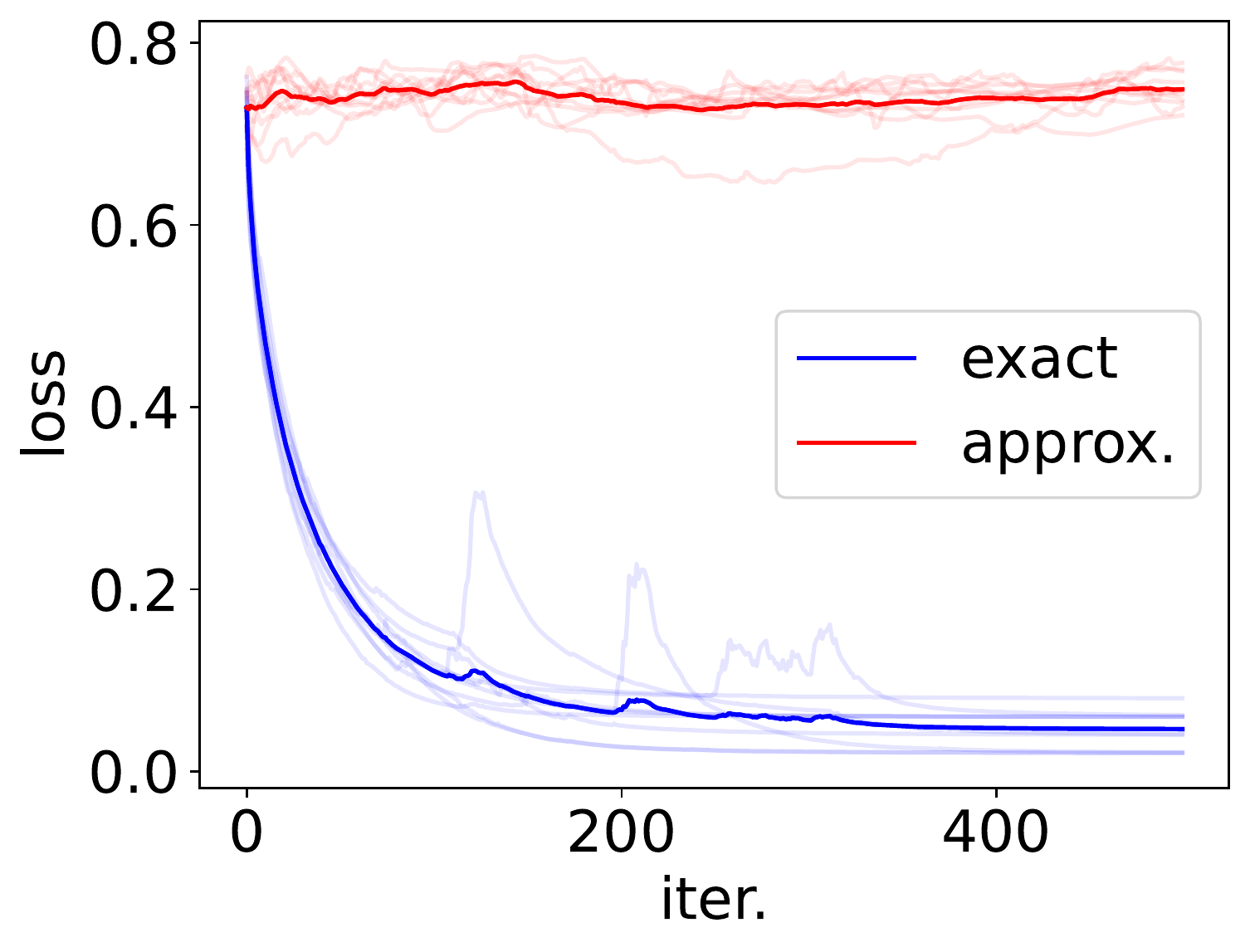} &
		\includegraphics[width=0.24\textwidth]{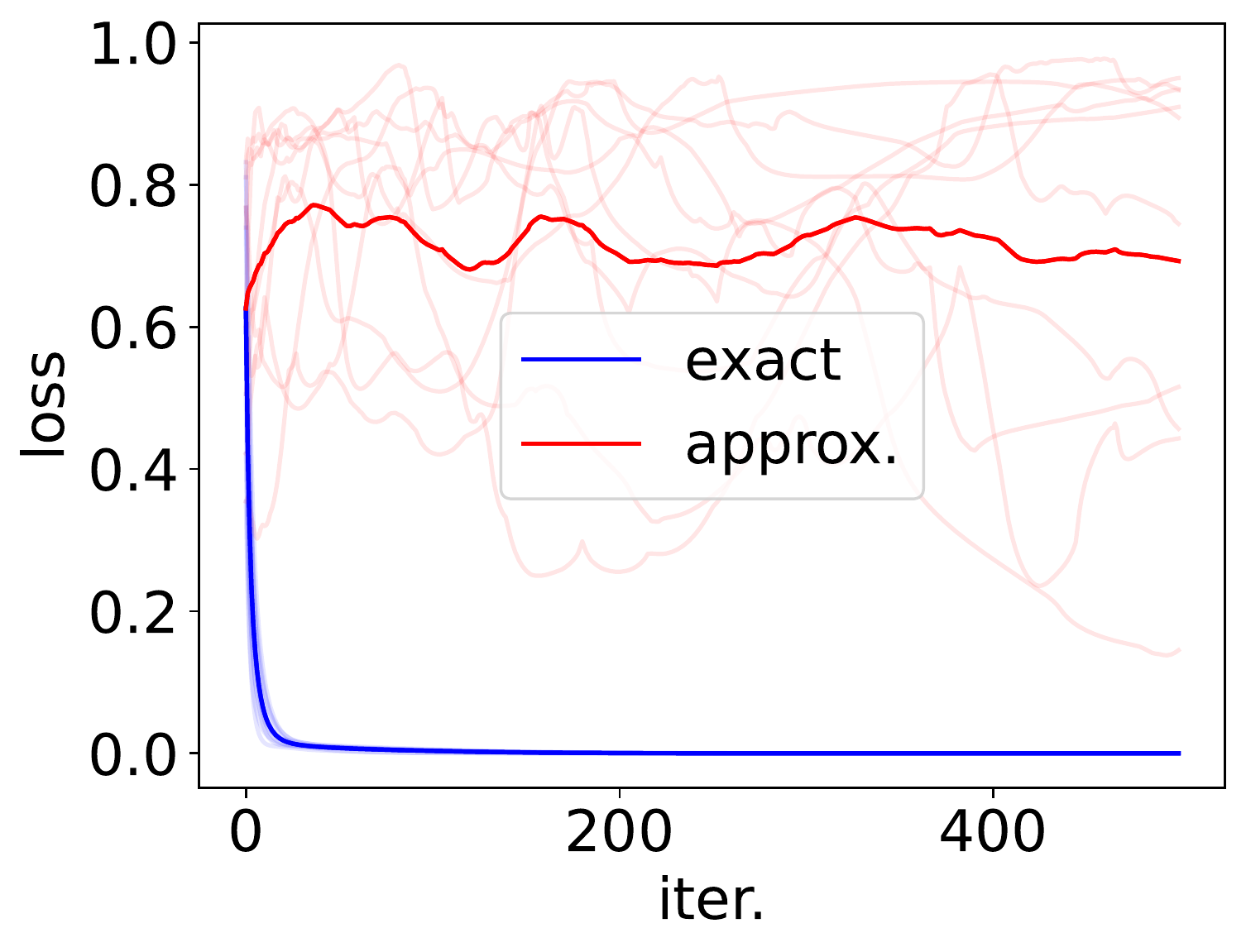} &
		\includegraphics[width=0.24\textwidth]{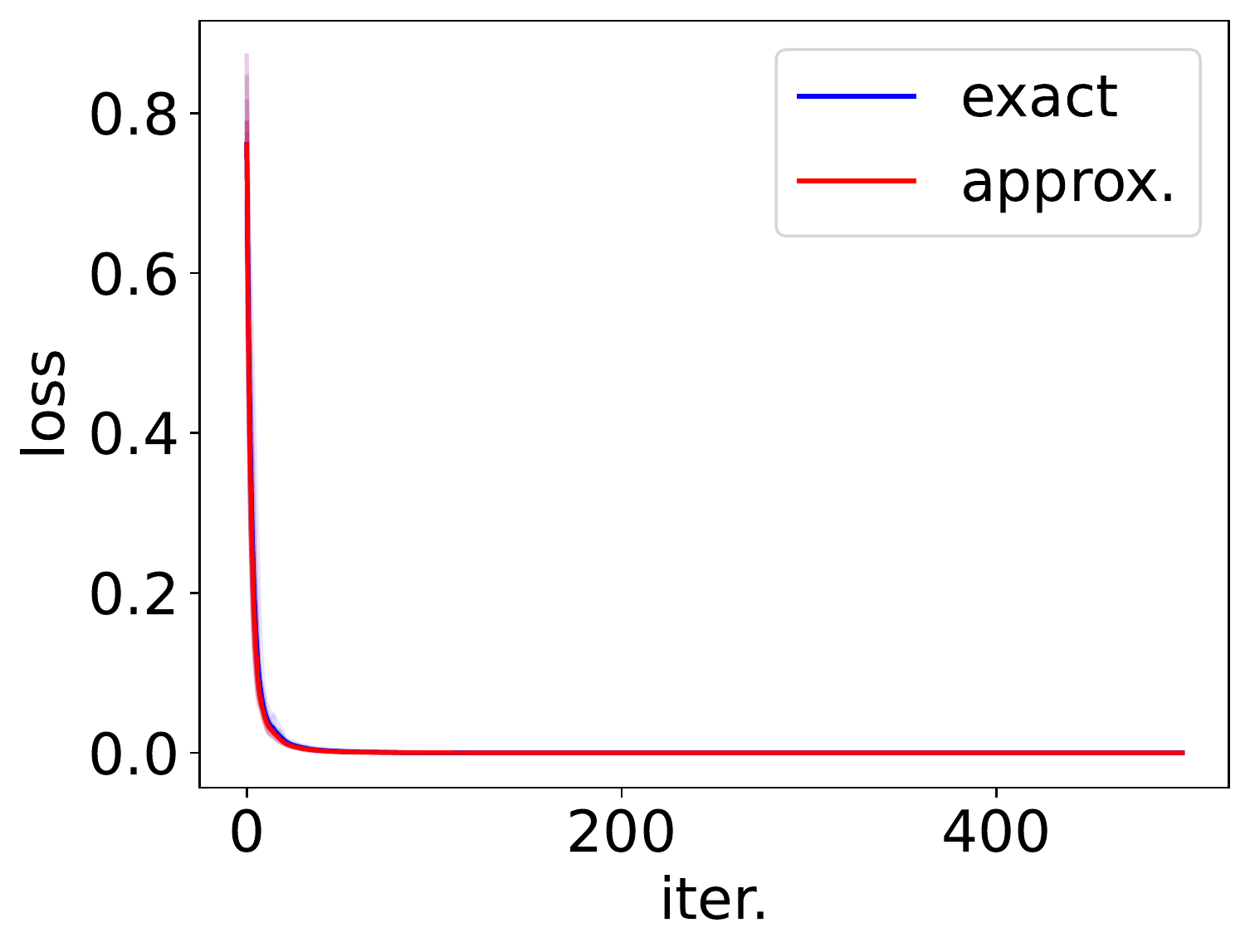} &
		\includegraphics[width=0.24\textwidth]{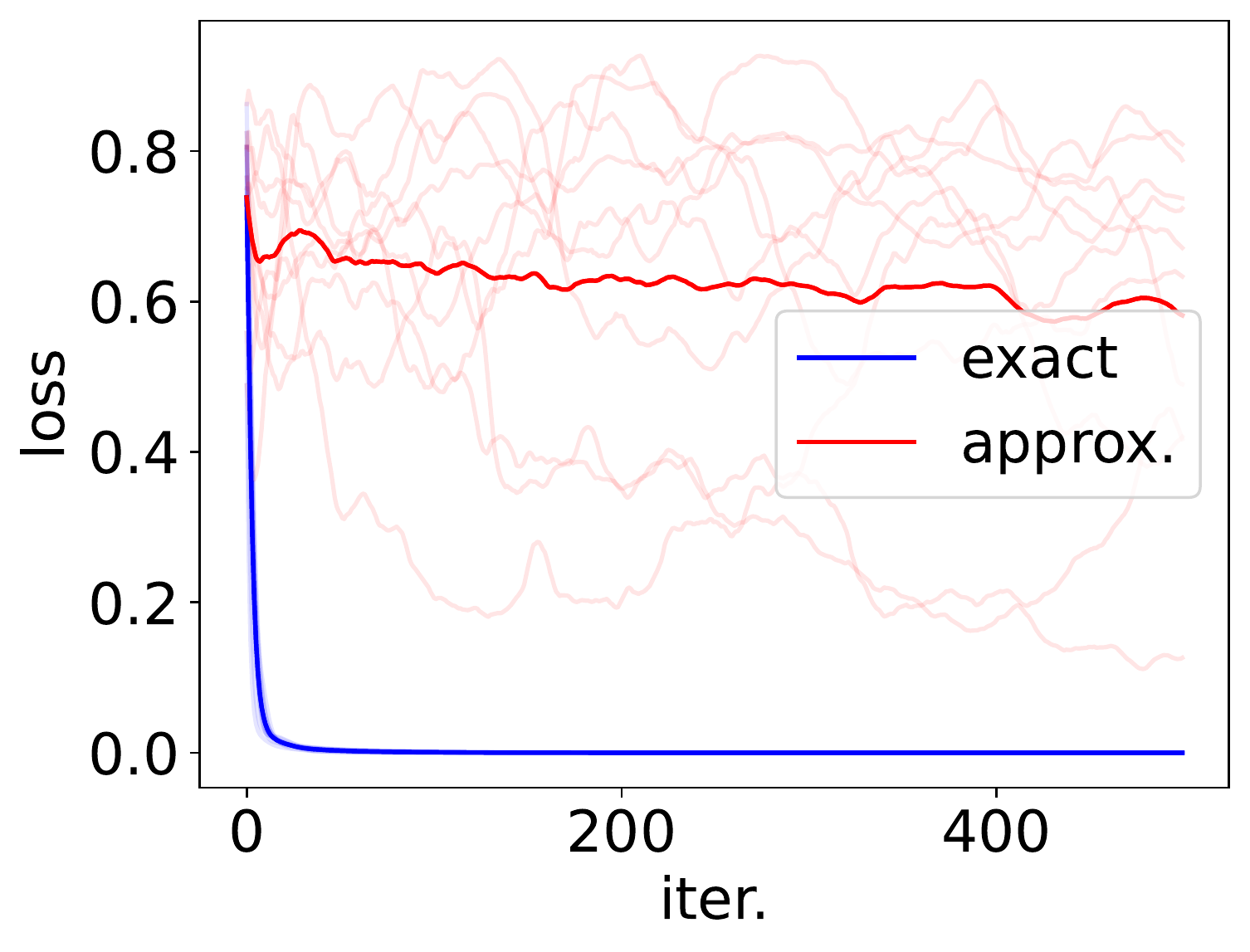}		
		\\
		\includegraphics[width=0.24\textwidth]{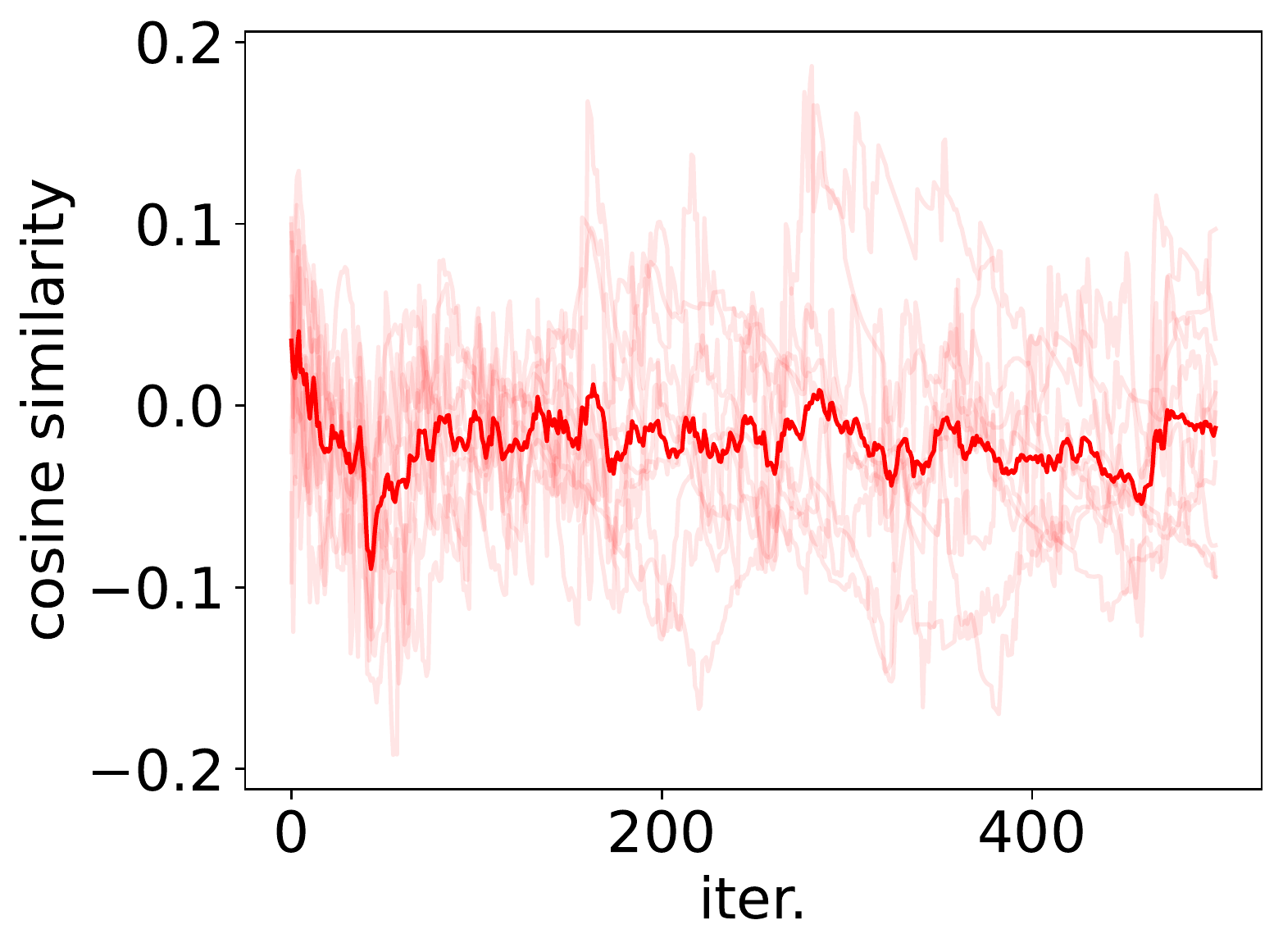} &
		\includegraphics[width=0.24\textwidth]{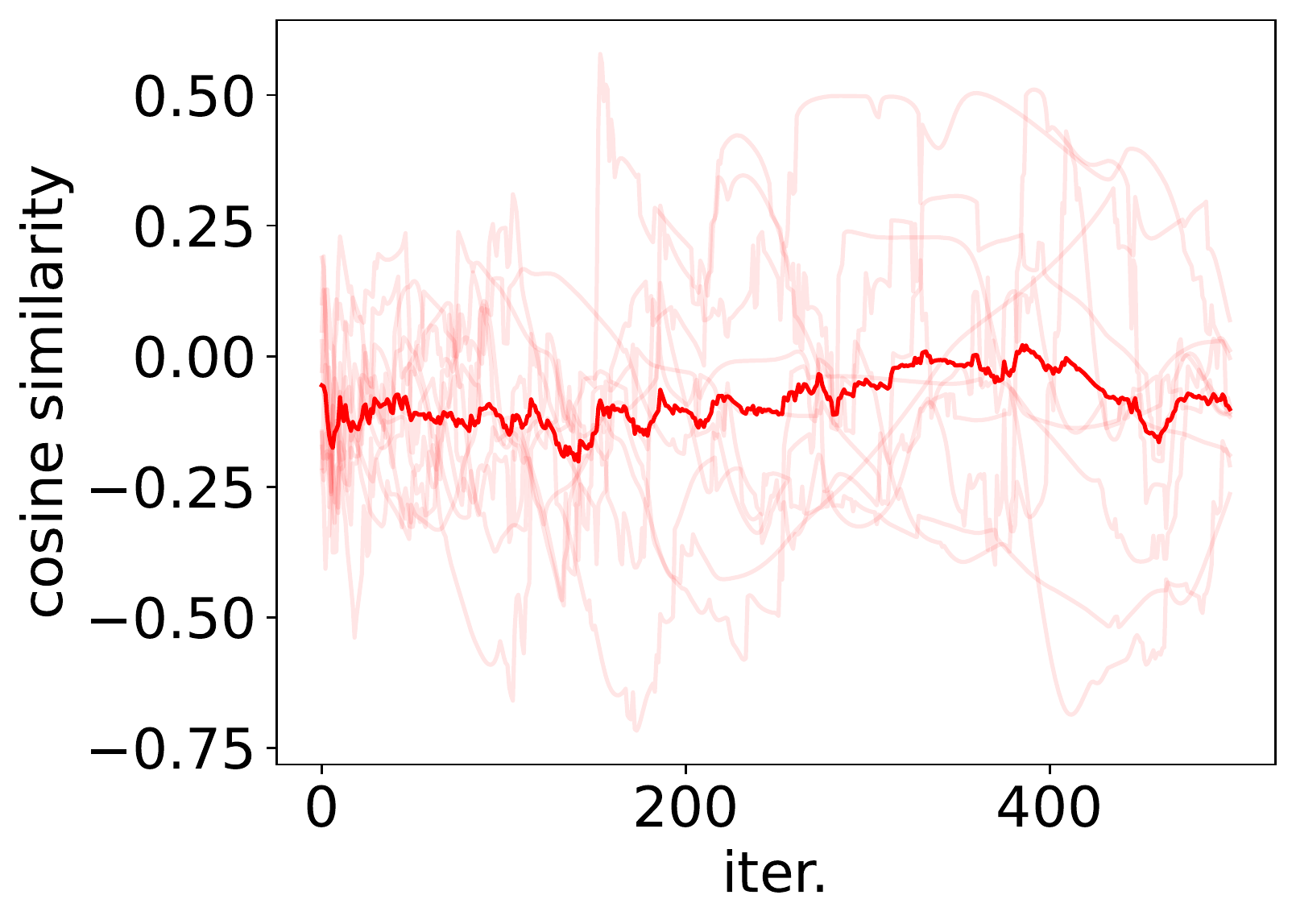} &
		\includegraphics[width=0.24\textwidth]{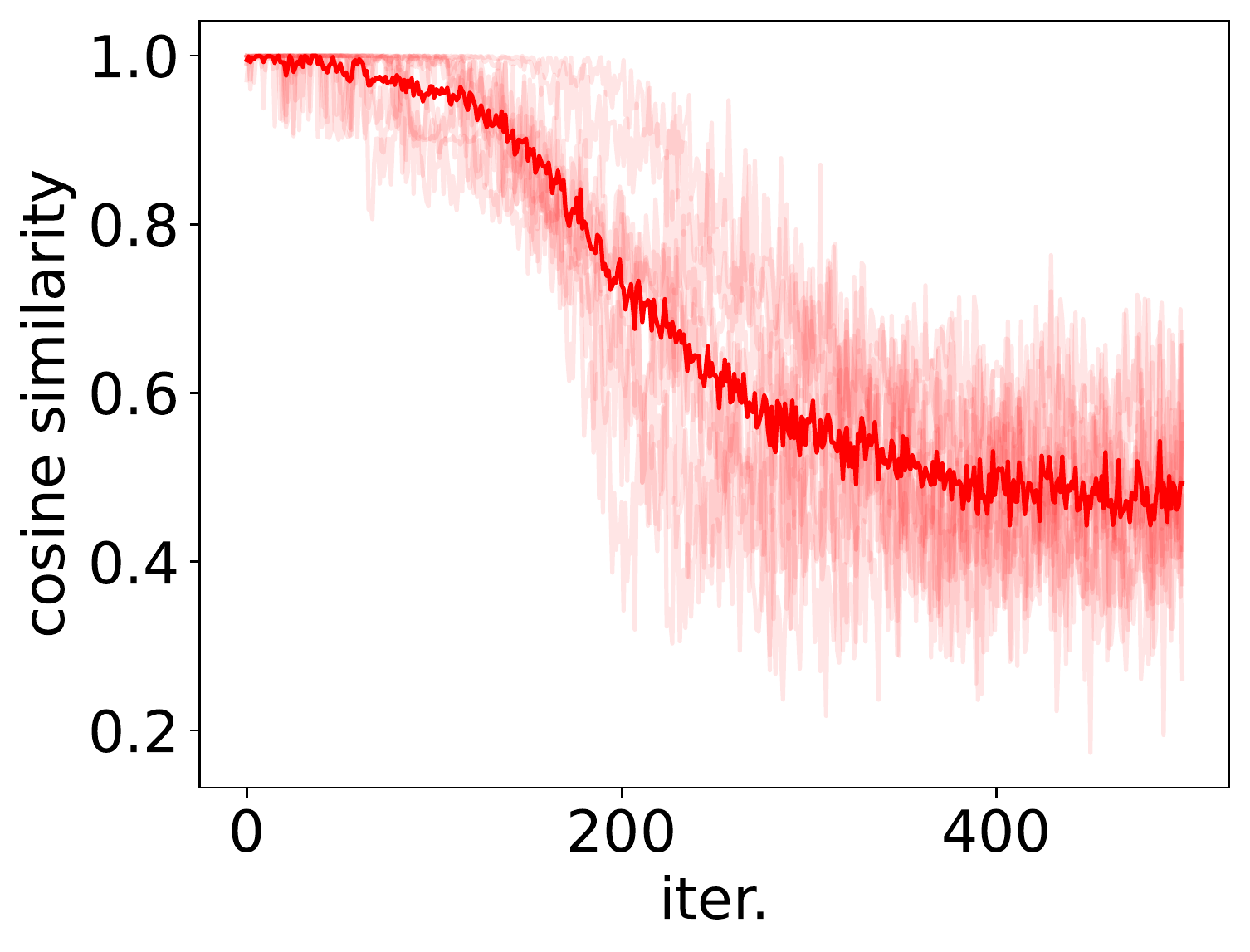} &
		\includegraphics[width=0.24\textwidth]{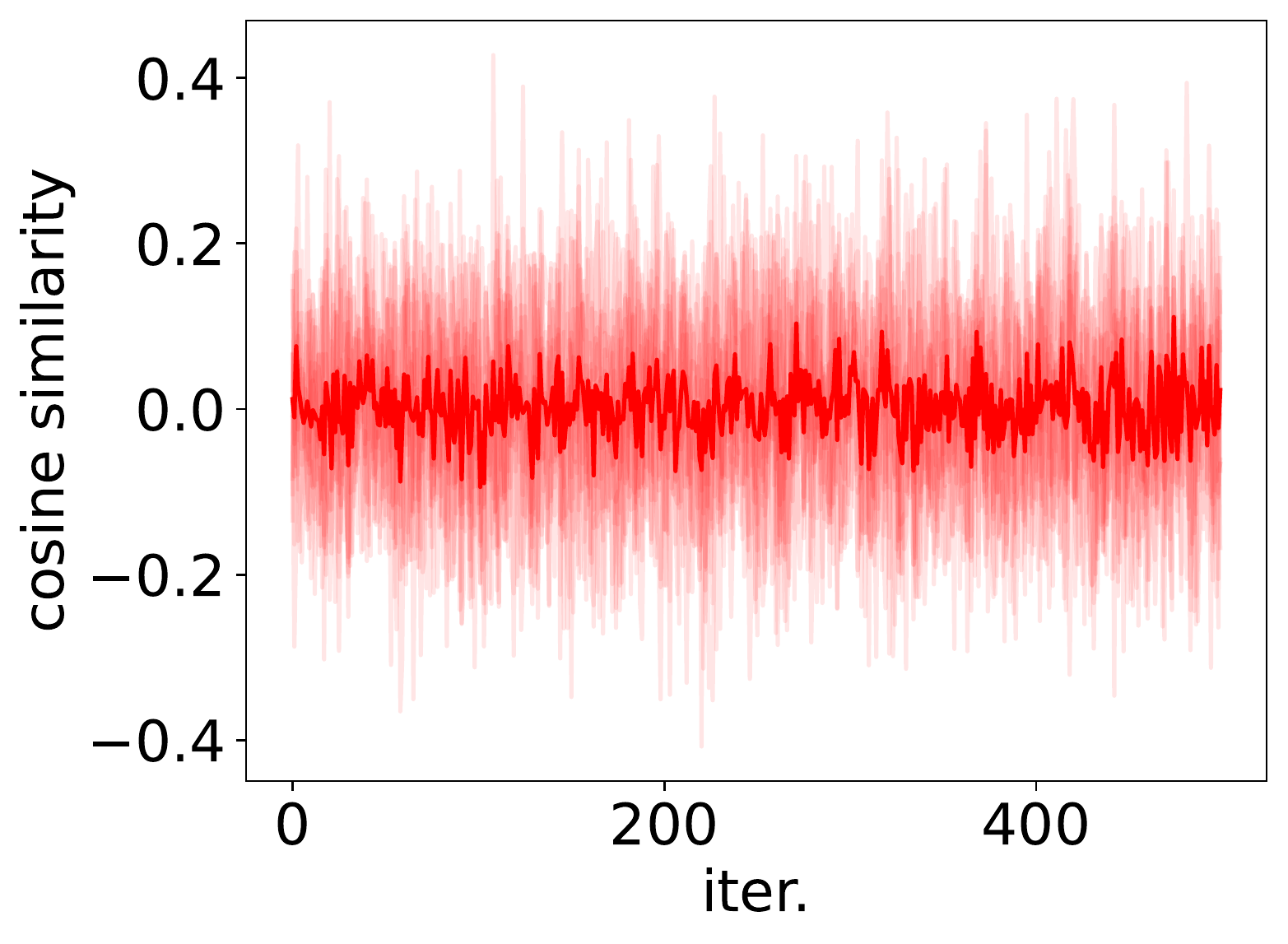}
		\\
		(a) all $\lambda$'s, general $X$ &
		(b) largest $\lambda$, general $X$ &
		(c) largest $\lambda$, $X \prec 0$ &
		(d) largest $\lambda$, rank-2 $X \succeq 0$
		\\
	\end{tabular}%
	\caption{Learning curves (top) and corresponding gradient cosine similarity (bottom) for eigen decomposition experiments. For (a) the loss is applied to all eigenvectors; for (b)--(d) it is only applied to the eigenvector corresponding to the largest eigenvalue.}
	\label{fig:ed}
\end{figure*}

\subsection{Eigen Decomposition}
\vspace{-1mm}

Given a real symmetric matrix $X = X\transpose \in \reals^{m \times m}$, the (unit) eigenvector associated with the largest eigenvalue of $X$ can be found by solving the following equality constrained optimization problem~\cite{Ghojogh:TR2019},
\begin{align}
	\begin{array}{ll}
		\text{maximize (over $u \in \reals^{m}$)} & u\transpose X u \\
		\text{subject to} & u\transpose u = 1.
	\end{array}
	\label{eqn:eigenprob}
\end{align}
Here we assume that the largest eigenvalue is simple otherwise a well-defined derivative does not exist. The optimality conditions for solution $y \in \reals^{m}$ are thus\footnote{Indeed, this holds for any simple eigenvalue-eigenvector pair.},
\begin{align}
	X y - \lambda_{\text{max}} y = 0_{m}
	\text{ and }
	y\transpose y = 1,
\end{align}
which gives differentials~\cite{Magnus:1985},
\begin{align}
%	\text{d} \lambda_{\text{max}} &= y\transpose (\text{d} X) y \\
	\text{d} y &= (\lambda_{\text{max}} I - X)^{\dagger} (\text{d}{X}) y
\end{align}
where ${}^{\dagger}$ denotes pseudo-inverse. %\footnote{Pseudo-inverse $(X - \lambda_{k} I)^{\dagger}$ can be computed efficiently given the full eigen decomposition of $X = Q \Lambda Q\transpose$ as $Q (\Lambda - \lambda_{k} I)^{\dagger} Q\transpose$.}
So with respect to the $(i,j)$-th component of $X$, and using symmetry, we have
\begin{align}
	\dd[X_{ij}] y(X) &= -\frac{1}{2} (X - \lambda_{\text{max}} I)^{\dagger} (y_j e_i + y_i e_j).
\end{align}
Ignoring the equality constraint $u\transpose u = 1$ we arrive at
\begin{align}
	\widehat{\dd[X_{ij}] y}(X) &= -\frac{1}{2} X^{\dagger} (y_j e_i + y_i e_j).
\end{align}
There is no computational gain here unless we need derivatives for multiple different eigenvectors and hence require multiple pseudo-inverses $(X - \lambda_k I)^{\dagger}$ for the exact gradient.

Moreover, results shown in \figref{fig:ed} confirm our analysis that the approximation is a poor choice, and rarely a descent direction, unless $y$ corresponds to the max.\ eigenvalue and all other eigenvalues are negative (equiv., the min.\ eigenvalue and all other eigenvalues are positive), as in \figref{fig:ed}(c).

\begin{figure}[b!]
	\centering \small
	\setlength{\tabcolsep}{2pt}
	\begin{tabular}{cc}
		\includegraphics[width=0.24\textwidth]{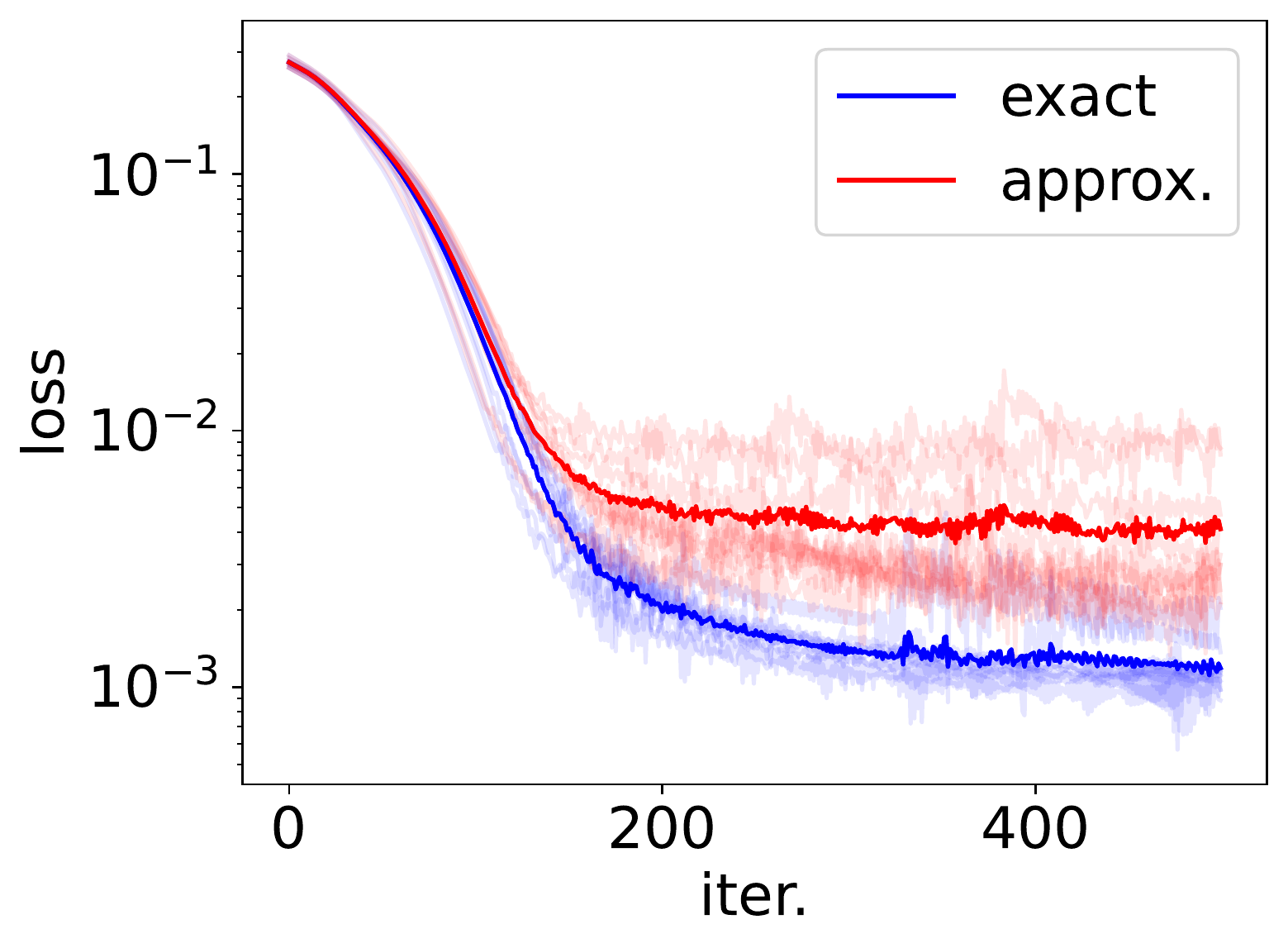} &
		\includegraphics[width=0.24\textwidth]{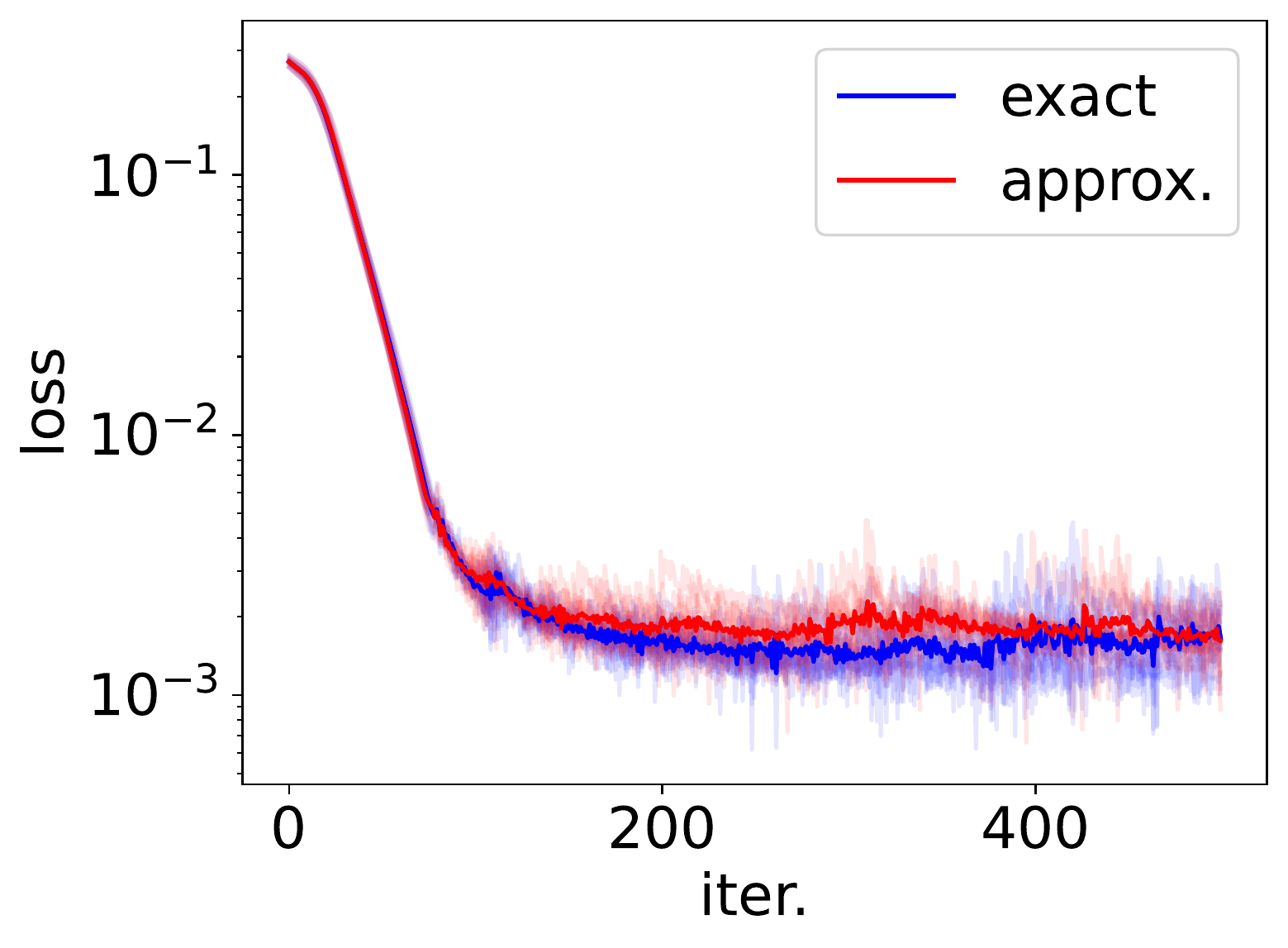} \\
		\includegraphics[width=0.24\textwidth]{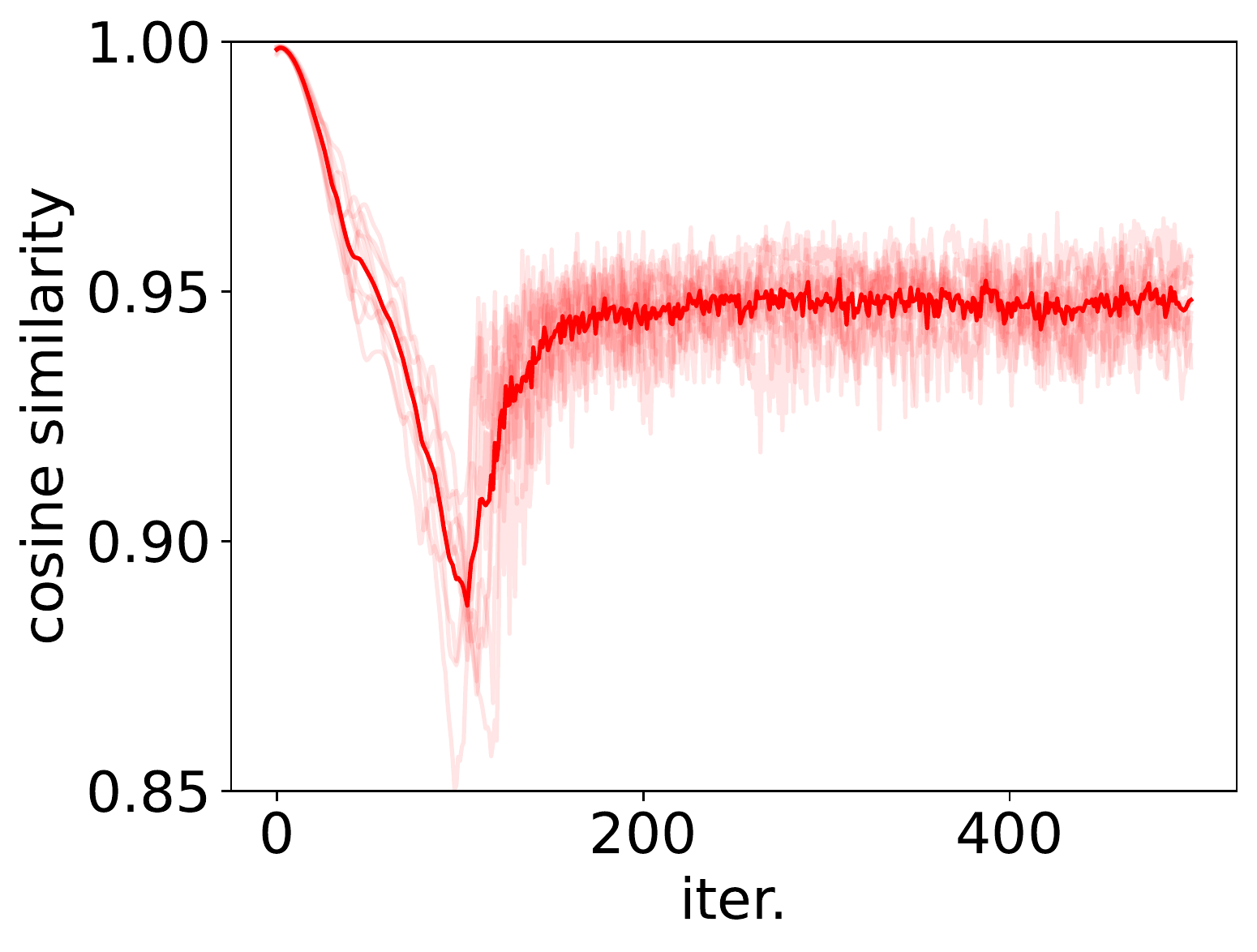} &
		\includegraphics[width=0.24\textwidth]{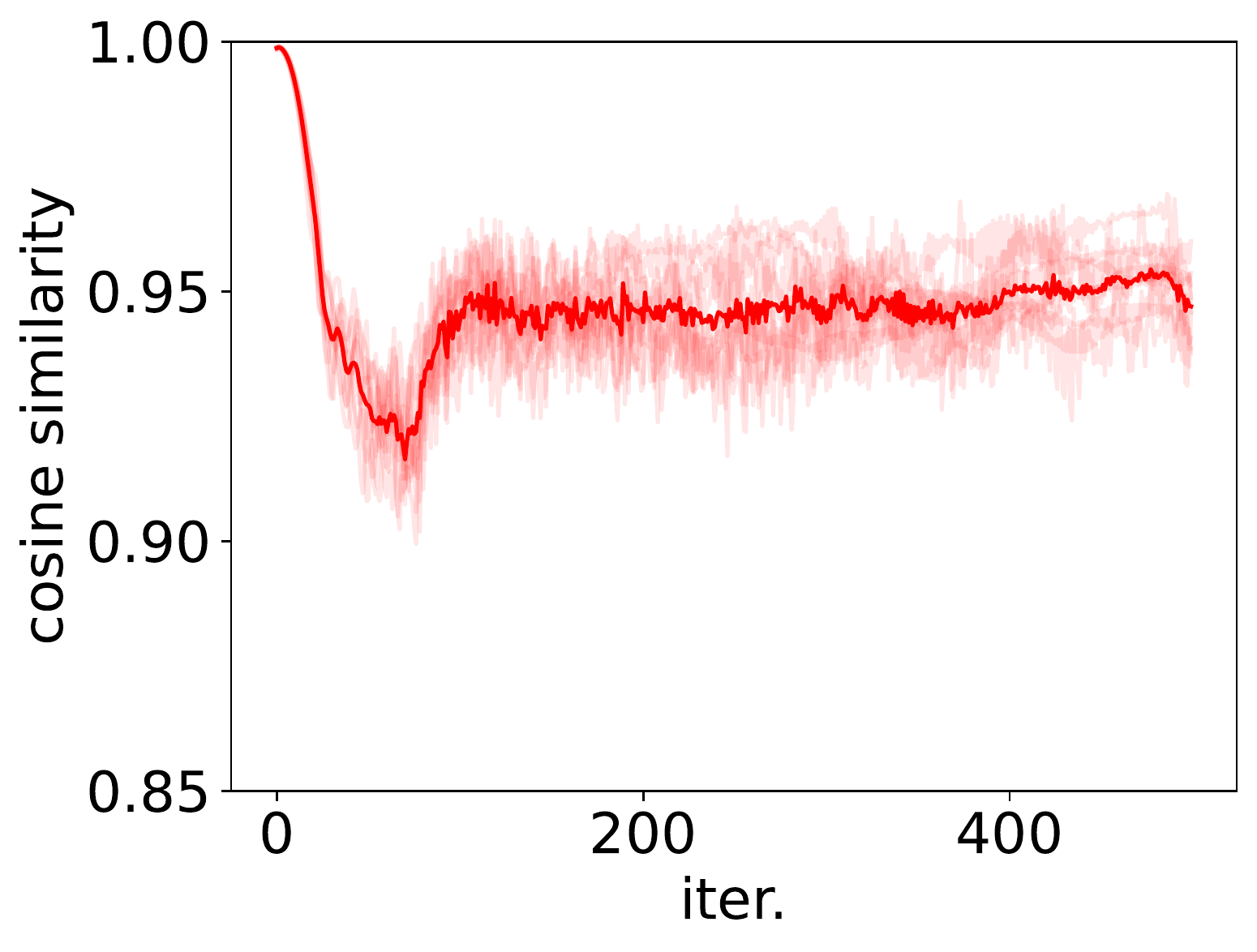} \\
		(a) under parameterized & (b) over parameterized
	\end{tabular}%
	\caption{Learning curves (top) and corresponding gradient cosine similarity (bottom) for optimal transport experiments.}
	\label{fig:ot}
\end{figure}

% --- Discussion --------------------------------------------------------------------------------

\vspace{-1mm}
\section{Discussion}
\label{sec:discussion}
\vspace{-1mm}

We have shown that (for certain objective functions) ignoring linear constraints gives a descent direction on average but that this does not always hold for normalization constraints. Experiments verify our analysis, and also show that even when we have a descent direction, the approximation tends to only work well in early stages of training. Whenever using approximations their behavior should be well-understood. This work is a step towards understanding of gradient approximations in differentiable optimization.

% --- REFERENCES ------------------------------------------------------------------------

\newpage
\bibliography{long,papers}
\bibliographystyle{icml2023}

%%%%%%%%%%%%%%%%%%%%%%%%%%%%%%%%%%%%%%%%%%%%%%%%%%%%%%%%%%%%%%%%%%%%%%%%%%%%%%%
%%%%%%%%%%%%%%%%%%%%%%%%%%%%%%%%%%%%%%%%%%%%%%%%%%%%%%%%%%%%%%%%%%%%%%%%%%%%%%%
% APPENDIX
%%%%%%%%%%%%%%%%%%%%%%%%%%%%%%%%%%%%%%%%%%%%%%%%%%%%%%%%%%%%%%%%%%%%%%%%%%%%%%%
%%%%%%%%%%%%%%%%%%%%%%%%%%%%%%%%%%%%%%%%%%%%%%%%%%%%%%%%%%%%%%%%%%%%%%%%%%%%%%%
\newpage
\appendix
\onecolumn
\section{Proofs and Derivations}

\subsection{Derivation of Equation~\ref{eqn:lindescentcond}}
\label{sec:lindescentcond}

Consider the function
$f(w) = w\transpose \left(I - \frac{a a\transpose H^{-1}}{a\transpose H^{-1}a} \right) w$.
If $w = 0$, then $f(w) = 0$. Otherwise,
\begin{align}
	f(w) \geq 0 &\iff \frac{1}{\|w\|_2^2} f(w) \geq 0.
\end{align}
But
\begin{align}
	\frac{1}{\|w\|_2^2} f(w)
	&= \frac{1}{\|w\|_2^2} w\transpose \left(I - \frac{a a\transpose H^{-1}}{a\transpose H^{-1}a} \right) w
	%&= \frac{w\transpose w}{\|w\|_2^2} - \frac{w\transpose}{\|w\|_2} \left(\frac{a a\transpose H^{-1}}{a\transpose H^{-1}a} \right) \frac{w}{\|w\|_2}
	= 1 - \frac{w\transpose}{\|w\|_2} \left(\frac{a a\transpose H^{-1}}{a\transpose H^{-1}a} \right) \frac{w}{\|w\|_2}
\end{align}
Therefore,
\begin{align}
	\min_{w} w\transpose \left(I - \frac{a a\transpose H^{-1}}{a\transpose H^{-1}a} \right) w \geq 0
	&\iff
	\max_{\|w\| = 1} \, w\transpose \left( \frac{a a\transpose H^{-1}}{a\transpose H^{-1}a} \right) w \leq 1.
\end{align}

\subsection{Proof of Proposition~\ref{prop:badnews}}
\label{sec:proof_badnews}

We begin with three useful lemmas for rank-1 quadratic forms, $f(x) = x\transpose \! \left(\frac{ab\transpose}{a\transpose b} \right) x = x\transpose \! \left(\frac{ab\transpose + ba\transpose}{2 a\transpose b} \right) x$ with $a \transpose b \neq 0$.

\begin{lemma}
	Let $M = ab\transpose + ba\transpose$. Then $M$ has eigenvalues $\lambda_{1,2} = a\transpose b \pm \|a\|\|b\|$ with corresponding orthonormal eigenvectors $q_{1,2} \propto \|b\| a \pm \|a\| b$.
	\label{lem:quad_eigen}
\end{lemma}
\begin{proof}
	By direct substitution,
	\begin{align}
		M q_1 &= (ab\transpose + ba\transpose) (\|b\| a + \|a\| b) \\
		&= ab\transpose (\|b\| a + \|a\| b) + ba\transpose (\|b\| a + \|a\| b) \\
		&= (b\transpose a + \|a\|\|b\|) \|b\| a + (\|b\|\|a\| + a\transpose b) \|a\| b \\
		&= (a\transpose b + \|a\|\|b\|) (\|b\| a + \|a\| b) \\
		&= \lambda_1 q_1
	\end{align}
	and similarly for $\lambda_2$ and $q_2$. We can verify orthogonality of $q_1$ and $q_2$ as
	\begin{align}
		q_1\transpose q_2 &= (\|b\|a + \|a\|b)\transpose (\|b\|a - \|a\|b) \\
		&= \|b\|^2\|a\|^2 - \|b\|\|a\|a\transpose b + \|a\|\|b\|b\transpose a - \|a\|^2\|b\|^2 \\
		&=0
	\end{align}
\end{proof}

\begin{lemma}
	The eigenvalue spectrum of $M = \frac{1}{2} \! \left(\frac{ab\transpose + ba\transpose}{a\transpose b}\right)$ with $a\transpose b \neq 0$ is
	\begin{align*}
		\sigma\left(\frac{ab\transpose + ba\transpose}{2 a\transpose b}\right)
		&= \left\{ \frac{1}{2} - \frac{\|a\|\|b\|}{2 |a\transpose b|},\, 0,\, \ldots,\, 0,\, \frac{1}{2} + \frac{\|a\|\|b\|}{2 |a\transpose b|} \right\}
	\end{align*}
	\label{lem:quad_eigen_spectrum}
\end{lemma}
\begin{proof}
	Follows from \lemref{lem:quad_eigen} taking careful note of the sign of $a\transpose b$. To see that all other eigenvalues are zero, note that $M$ is a rank-2 matrix (rank-1 if $a$ and $b$ are linearly dependent) and so has at most two non-zero eigenvalues.
\end{proof}

It follows also that if $a$ and $b$ are linearly dependent then $\frac{1}{2} \! \left(\frac{ab\transpose + ba\transpose}{a\transpose b}\right)$ has a single non-zero eigenvalue of 1. Moreover, for any non-orthogonal $a$ and $b$, the sum of eigenvalues is equal to one.\\

\begin{lemma}
	Let $a, b \in \reals^n$ with $a\transpose b \neq 0$. Then
	\begin{align*}
		\max_{\|x\| = 1} \, x\transpose \left(\frac{ab\transpose}{a\transpose b} \right) x
		&= \frac{1}{2} + \frac{\|a\|\|b\|}{2 |a\transpose b|}.
	\end{align*}
	\label{lem:quad_max}
\end{lemma}
\begin{proof}
	We have
	\begin{align}
		\max_{\|x\| = 1} \, x\transpose \left(\frac{ab\transpose}{a\transpose b} \right) x
		&=
		\max_{\|x\| = 1} \, x\transpose \left(\frac{ab\transpose + ba\transpose}{2a\transpose b} \right) x
		=
		\lambda_{\text{max}}\left( \frac{ab\transpose + ba\transpose}{2a\transpose b} \right)
		= \frac{1}{2} + \frac{\|a\|\|b\|}{2 |a\transpose b|}
	\end{align}
	by \lemref{lem:quad_eigen_spectrum}.
\end{proof}

We are now ready to prove \propref{prop:badnews}. From \lemref{lem:quad_max} with $b = H^{-1}a$ we have
\begin{align}
	\max_{\|w\| = 1} \, w\transpose \left(\frac{aa\transpose H^{-1}}{a\transpose H^{-1} a}\right) w
	&= \frac{1}{2} + \frac{\|a\| \cdot \|H^{-1} a\|}{2 |a\transpose H^{-1} a|}
\end{align}

By the Cauchy-Schwarz inequality $|a\transpose H^{-1} a| \leq \|a\| \cdot \|H^{-1} a\|$ with equality if and only if $a$ and $H^{-1} a$ are linearly dependent. This gives the lower bound,
\begin{align}
	\max_{\|w\| = 1} \, w\transpose \left(\frac{aa\transpose H^{-1}}{a\transpose H^{-1} a}\right) w &\geq 1.
\end{align}
For the upper bound, observe that $|a\transpose H^{-1} a| \geq \sigma_{\text{min}}(H^{-1}) \|a\|^2$ and $\|H^{-1} a\| \leq \sigma_{\text{max}}(H^{-1})\|a\|$ to get
\begin{align}
	\frac{\|a\| \cdot \|H^{-1} a\|}{|a\transpose H^{-1} a|}
	&\leq \frac{\|a\| \cdot \sigma_{\text{max}}(H^{-1}) \|a\|}{\sigma_{\text{min}} (H^{-1}) \|a\|^2}
	= \frac{\sigma_{\text{max}}(H^{-1})}{\sigma_{\text{min}}(H^{-1})}
	= \text{cond}(H).
\end{align}

\subsection{Proof of Proposition~\ref{prop:single_lin_expect}}
\label{sec:proof_single_lin_expect}

The following general result on the expected value for a quadratic form is from \citet{Seber:2003}[Thm.~1.5, p.~9]. It can be easily proved by direct evaluation, using the cyclic property of trace, linearity of trace and expectation, and definition of variance.

\begin{lemma}
	\label{lem:gaussian_expectation}
	\textup{\cite{Seber:2003}.}
	Let $X = (x_i)$ be an $n \times 1$ vector of random variables, and let $A$ be an $n \times n$ symmetric matrix. If $\expectation{X} = \mu$ and $\textup{Var}(X) = \Sigma = (\sigma_{ij})$, then
	\begin{align*}
		\expectation{x\transpose \! A x} &= \trace{A\Sigma} + \mu\transpose \! A \mu.
	\end{align*}
\end{lemma}

%\begin{proof}
%	By direct evaluation, using the cyclic property of trace, linearity of trace and expectation, and definition of variance,
%	%
%	\begin{multline}
%		\expectation{x\transpose \! A x} = \expectation{\trace{x\transpose \! A x}}
%		= \expectation{\trace{A x x\transpose}}
%		= \trace{A \expectation{x x\transpose}}
%		= \trace{A (\Sigma + \mu\mu\transpose)} \\
%		= \trace{A \Sigma} + \trace{A \mu \mu\transpose}
%		= \trace{A \Sigma} + \trace{\mu\transpose \! A \mu}
%		= \trace{A \Sigma} + \mu\transpose \! A \mu.
%	\end{multline}
%\end{proof}

The above result extends to nonsymmetric $A$ since, $x\transpose A x = x\transpose \frac{1}{2}(A + A\transpose) x$ and $\trace{A \Sigma} = \trace{\Sigma A\transpose} = \trace{A\transpose \Sigma}$ so that
\begin{align}
	\trace{A \Sigma} &= \frac{1}{2} \left( \trace{A \Sigma} + \trace{A\transpose \Sigma} \right)
	= \trace{\frac{1}{2}(A + A\transpose) \Sigma}.
\end{align}

Now assuming the quantity $w$ in
\begin{align}
	g\transpose \widehat{g} &= w\transpose \left(I - \frac{a a\transpose H^{-1}}{a\transpose H^{-1}a} \right) w
\end{align}
is isotropic Gaussian distributed, then
\begin{align}
	\expectwrt[w \sim \N(0, I)]{w\transpose \! \left(I - \frac{aa\transpose H^{-1}}{a\transpose H^{-1} a}\right) \! w}
	&= \trace{I - \frac{aa\transpose H^{-1}}{a\transpose H^{-1} a}} \\
	&= \trace{I} - \trace{\frac{aa\transpose H^{-1}}{a\transpose H^{-1} a}} \\
	&= m - 1
\end{align}
where the first line is from \lemref{lem:gaussian_expectation}, the second line is by linearity of trace, and the last line is by the trace of a matrix equalling the sum of its eigenvalues, which is $m$ for the identity and one for the second term by \lemref{lem:quad_eigen_spectrum}.

\subsection{Proof of Proposition~\ref{prop:multi_lin_expect}}
\label{sec:proof_multi_lin_expect}

From \lemref{lem:gaussian_expectation} we have
\begin{align}
	\expectwrt[w \sim \N(0, I)]{w\transpose \! \left(I - A\transpose \! \left(A H^{-1} A\transpose\right)^{-1} \! A H^{-1}\right) w}
	&= \trace{I - A\transpose \! \left(A H^{-1} A\transpose\right)^{-1} \! A H^{-1}}
	\\
	&= \trace{I} - \trace{A\transpose \left(A H^{-1} A\transpose\right)^{-1} \! A H^{-1}}
	\\
	&= \trace{I_m} - \trace{\left(A H^{-1} A\transpose\right)^{-1} \! A H^{-1} A\transpose}
	\\
	&= \trace{I_m} - \trace{I_p}
	\\
	&= m - p
\end{align}
where we have used the cyclic property of trace on the third line, and that $A$ is full rank and $H \succ 0$ on the fourth line.

\subsection{Proof of Proposition~\ref{prop:norm_expect}}
\label{sec:proof_norm_expect}

Let $\widehat{H} = Q \Lambda Q\transpose$ where $\Lambda = \diag{\lambda_1, \ldots, \lambda_m}$ is a diagonal matrix containing the eigenvalues of $\widehat{H}$ arranged in ascending order. Then $H = Q (\Lambda - \lambda I) Q\transpose$.
Since $\widehat{H}$ and $H$ share the same eigenvectors they are simultaneously diagonalizable and so commute. Therefore
\begin{align}
	v\transpose \left(H^{-1} - \frac{H^{-1} y y\transpose H^{-1}}{y\transpose H^{-1} y}\right) \widehat{H}^{-1} v
	&= v\transpose \left(H^{-1}\widehat{H}^{-1} - \frac{H^{-1} y y\transpose H^{-1} \widehat{H}^{-1}}{y\transpose H^{-1} y}\right) v \\
	&= v\transpose \left(H^{-1}\widehat{H}^{-1} H H^{-1} - \frac{H^{-1} y y\transpose \widehat{H}^{-1} H^{-1}}{y\transpose H^{-1} y}\right) v \\
	&= v\transpose H^{-1} \left(\widehat{H}^{-1} H - \frac{y y\transpose \widehat{H}^{-1}}{y\transpose H^{-1} y}\right) H^{-1} v \\
	&= w\transpose \left(\widehat{H}^{-1} H - \frac{y y\transpose \widehat{H}^{-1}}{y\transpose H^{-1} y}\right) \! w
\end{align}
where in the second line we have used that $\widehat{H}$ and $H$ commute in the second term, then factored out $H^{-1}$ in the third line, and substituted $w = H^{-1} v$ in the last line. As for the linear equality constrained case, we can compute expectations,
\begin{align}
	\expectwrt[w \sim \N(0, I)]{g\transpose \widehat{g}} 
	&= \expectwrt[w \sim \N(0, I)] {w\transpose \! \left(\widehat{H}^{-1} H - \frac{y y\transpose \widehat{H}^{-1}}{y\transpose H^{-1} y}\right) \! w} \\
	&= \trace{\widehat{H}^{-1} H - \frac{y y\transpose \widehat{H}^{-1}}{y\transpose H^{-1} y}}
	\\
	&= \trace{\widehat{H}^{-1} H} - \trace{\frac{y y\transpose \widehat{H}^{-1}}{y\transpose H^{-1} y}}
	\\
	&= \trace{Q\Lambda^{-1} Q\transpose Q(\Lambda - \lambda I)Q\transpose} - \trace{\frac{y\transpose \widehat{H}^{-1} y}{y\transpose H^{-1} y}}
	\\
	&= \trace{\Lambda^{-1} (\Lambda - \lambda I)} - \trace{\frac{y\transpose \widehat{H}^{-1} y}{y\transpose H^{-1} y}}
	\\
	&= \sum_{i=1}^{m} \frac{\lambda_i - \lambda}{\lambda_i} - \frac{y\transpose \widehat{H}^{-1} y}{y\transpose H^{-1} y}
\end{align}
where $\lambda_1 \leq \lambda_2 \leq \cdots \leq \lambda_m$ are the eigenvalues of $\widehat{H}$ and 
$\lambda_1 - \lambda \leq \lambda_2 - \lambda \leq \cdots \leq \lambda_m - \lambda$ are the eigenvalues of $H$.

\subsection{Proof of Proposition~\ref{prof:norm_expect_cases}}
\label{sec:proof_norm_expect_cases}

We consider each case.

{\bf Case 1 ($\widehat{H} \succ 0$, $\lambda < \lambda_1$).} In this case $H$ is positive definite.
Write $\widehat{H} = Q \Lambda Q\transpose$ and $H = Q (\Lambda - \lambda I)Q\transpose$. Then
\begin{align}
	\min_{\|y\| = 1} \, -\frac{y\transpose \widehat{H}^{-1} y}{y\transpose H^{-1} y}
	&= -\max_{\|y\| = 1} \, \frac{y\transpose \widehat{H}^{-1} y}{y\transpose H^{-1} y} \\
	&= -\max_{\|y\| = 1} \, y\transpose H^{1/2}\widehat{H}^{-1}H^{1/2} y  \\
	&= -\max_{\|y\| = 1} \, y\transpose \left(Q(\Lambda - \lambda I)^{1/2} Q\transpose\right) Q\Lambda^{-1}Q\transpose \left(Q(\Lambda - \lambda I)^{1/2} Q\transpose\right) y  \\	
	&= -\max_{\|y\| = 1} \, y\transpose (\Lambda - \lambda I)^{1/2} \Lambda^{-1} (\Lambda - \lambda I)^{1/2} y  \\	
	&= -\max_{i = 1, \ldots, m} \left\{ \frac{\lambda_i - \lambda}{\lambda_i} \right\} \\
	&= \begin{cases}
		-\frac{\lambda_m - \lambda}{\lambda_m}, & \text{if $\lambda \geq 0$} \\
		-\frac{\lambda_1 - \lambda}{\lambda_1}, & \text{otherwise.}
	\end{cases}
\end{align}
Therefore,
\begin{align}
	\expectwrt[w \sim \N(0, I)] {g\transpose \widehat{g}}
	&\geq \sum_{i=1}^{m} \frac{\lambda_i - \lambda}{\lambda_i} - \max_{i = 1, \ldots, m} \left\{ \frac{\lambda_i - \lambda}{\lambda_i} \right\} \\
	&= \begin{cases}
		\sum_{i=1}^{m-1} \frac{\lambda_i - \lambda}{\lambda_i}, & \text{if $\lambda \geq 0$} \\
		\sum_{i=2}^{m} \frac{\lambda_i - \lambda}{\lambda_i}, & \text{otherwise}
	\end{cases} \\
	&\geq 0
\end{align}
since each $\frac{\lambda_i - \lambda}{\lambda_i}$ is positive.

{\bf Case 2 ($\widehat{H} \succ 0$, $\lambda > \lambda_m$).} In this case $H$ is negative definite, and we have
\begin{align}
	\max_{\|y\| = 1} \, -\frac{y\transpose \widehat{H}^{-1} y}{y\transpose H^{-1} y}
	&= \max_{\|y\| = 1} \, \frac{y\transpose \widehat{H}^{-1} y}{y\transpose (-H^{-1}) y} \\
	&= \max_{\|y\| = 1} \, y\transpose (-H)^{1/2}\widehat{H}^{-1}(-H)^{1/2} y  \\
	&= \max_{\|y\| = 1} \, y\transpose \left(Q(\lambda I - \Lambda)^{1/2} Q\transpose\right) Q\Lambda^{-1}Q\transpose \left(Q(\lambda I - \Lambda)^{1/2} Q\transpose\right) y  \\	
	&= \max_{\|y\| = 1} \, y\transpose (\lambda I - \Lambda)^{1/2} \Lambda^{-1} (\lambda I - \Lambda)^{1/2} y  \\	
	&= \max_{i = 1, \ldots, m} \left\{ \frac{\lambda - \lambda_i}{\lambda_i} \right\} \\
	&= \frac{\lambda - \lambda_1}{\lambda_1}.
\end{align}
Therefore,
\begin{align}
	\expectwrt[w \sim \N(0, I)] {g\transpose \widehat{g}}
	&\leq \sum_{i=1}^{m} \frac{\lambda_i - \lambda}{\lambda_i} + \frac{\lambda - \lambda_1}{\lambda_1} \\
	&= \sum_{i=2}^{m} \frac{\lambda_i - \lambda}{\lambda_i} \\
	&\leq 0
\end{align}
since each $\frac{\lambda_i - \lambda}{\lambda_i}$ is negative.

\section{Experimental Details}
\label{sec:experimental_details}

We follow the same experimental procedure for all three example optimization problems---Euclidean projection onto the unit sphere, optimal transport, and eigen decomposition---as depicted in \figref{fig:experimental_setup}. The network architecture consists of a three-layer multi-layer perceptron (MLP) with ReLU activation layers. The MLP maps $d$-dimensional raw input data $z_b$ into the input for a deep declarative node (also known as a differentiable optimization layer) denoted $x_b$. This is an $n$-dimensional vector for Euclidean projection, an $m$-by-$n$ dimensional matrix for optimal transport (for simplicity we set $n = m$), and an $m$-by-$m$ real symmetric matrix for eigen decomposition. Using this input the declarative node solves the associated optimization problem, outputting the solution $y_b$ corresponding to $z_b$. As such the output of the network $y_b$ can be thought of as a function of input $x_b$ and MLP parameters $\theta$.

A single batch of ten input-target pairs $\{(z_b, y^\star_b)\}_{b=1}^{10}$ is randomly generated and used as training data for the parameters $\theta$ of the MLP. We use the AdamW optimizer~\cite{Loshchilov:ICLR19} with a learning rate of $10^{-3}$ and run for a total of 500 iterations. The loss function of Euclidean projection and optimal transport is the mean-square-error,
\begin{align}
	\Ell(\theta) &= \frac{1}{B} \sum_{b=1}^{B} \|y_b(z_b, \theta) - y^\star_b\|_2^2 
\end{align}
whereas for eigen decomposition we use the mean absolute-value of the cosine similarity,
\begin{align}
	\Ell(\theta) &= \frac{1}{B} \sum_{b=1}^{B} |y_b(z_b, \theta)\transpose y^\star_b|.
\end{align}
The latter allows us to seamlessly deal with the sign ambiguity of eigenvectors (i.e., if $q$ is a unit eigenvector then so is $-q$). 

We run five repeats of each experiment, randomly resampling the training data for each run. Learning curve plots show the loss function versus training iteration for each individual run (light) and the average over all five runs (dark). 

For Euclidean projection and optimal transport we include two different input settings, $d=5$ and $d=100$, which we denote as under and over parameterized in the plots. This reflects the fact that it is easier to learn a mapping from high-dimensional input $z_b$ to arbitrary target $y_b^\star$ than for low-dimensional $z_b$. We set $m=10$ for both problems. For eigen decomposition we experiment with four different settings: (a) a loss on all eigenvectors $y_k$ for a general input matrix $X$, (b) a loss on just the eigenvector corresponding to the maximum eigenvalue for general input matrix $X$, (c) the same loss but with negative definite input matrix, and (d) the same loss but with a rank-2 positive definite input matrix. In all cases we set $d=5$ and $m=10$.

In addition to the learning curves we plot the cosine-similarity between the true and approximate gradient of the loss with respect to the input of the declarative node $x_b$. This is done for each point on the learning curve for approximate gradient. A value greater than zero indicates that the corresponding approximate gradient is a descent direction with respect to $x_b$. We note that this does not necessarily mean that it is a descent direction with respect to the parameters $\theta$ of the MLP, which depends on the structure of $\dd[\theta] x_b$. In other words, a descent direction for $x_b$ does not guarantee a descent direction for $\theta$.

All experiments were run using PyTorch 1.13.0~\cite{PyTorch} on an Intel i7-8565U CPU @ 1.80GHz. Full source code available at \url{http://deepdeclarativenetworks.com}.

\end{document}